\documentclass{article}

    \usepackage[nonatbib, preprint]{neurips_2021}

\usepackage[utf8]{inputenc} 
\usepackage[T1]{fontenc}    
\usepackage{hyperref}       
\usepackage{url}            
\usepackage{booktabs}       
\usepackage{amsfonts}       
\usepackage{nicefrac}       
\usepackage{microtype}      
\usepackage{xcolor}         
\usepackage{csquotes}
\usepackage{amsmath}
\usepackage{nicefrac}       
\usepackage{wrapfig}
\usepackage{graphicx}
\usepackage{hyperref}
\usepackage{url}            
\usepackage{thmtools}
\usepackage{thm-restate}
\usepackage{mathtools}
\usepackage{amsthm}
\usepackage[numbers,sort]{natbib}

\def\NN{{\mathbb N}}    
\def\RR{{\mathbb R}}    
\def\PP{{\mathbb P}}     
\def\EE{{\mathbb E}}    
\def\11{{\mathbf 1}}    

  \def\cG{{\mathcal G}}  \def\cS{{\mathcal S}}   \def\cH{{\mathcal H}} \def\cN{{\mathcal N}}     \def\cO{{\mathcal O}}  \def\cD{{\mathcal D}}   \def\cP{{\mathcal P}}  \def\cE{{\mathcal E}}    \def\cW{{\mathcal W}} \def\cF{{\mathcal F}}  \def\cL{{\mathcal L}}  \def\cX{{\mathcal X}} \def\cY{{\mathcal Y}}  \def\cZ{{\mathcal Z}}

\def\boldx{{\boldsymbol{x}}} \def\boldt{{\boldsymbol{t}}}
\def\bfx{{\bf x}} \def\bfy{{\bf y}} \def\bfz{{\bf z}} \def\bfw{{\bf w}}
 \def\bfK{{\bf K}} 
\def\bfL{{\bf L}} \def\bfQ{{\bf Q}} \def\bfA{{\bf A}}
\def\bfPhi{{\bf \Phi}} \def\bfPsi{{\bf \Psi}}
\def\boldupsilon{\boldsymbol{\Upsilon}}
\def\bfeta{\boldsymbol{\eta}} \def\bfSigma{\boldsymbol{\Sigma}}

\def\rmc{{\mathrm c}}
\def\rmd{{\mathrm d}}

\newcommand{\deffunction}[5]{
{#1}:
\left|
  \begin{array}{rcl}
    {#2} & \longrightarrow & {#3} \\
    {#4} & \longmapsto & {#5} \\
  \end{array}
\right.
}

\def\d{\,{\mathrm d}}

\def\tr{\operatorname{tr}}
\newcommand{\Range}[1]{\operatorname{Range}({#1})}

\newcommand{\cSpan}[1]{\overline{\operatorname{Span}}\left\{{#1}\right\}}
\theoremstyle{plain}
	\newtheorem{theorem}{Theorem}[section] 
	\newtheorem{proposition}[theorem]{Proposition}        

\theoremstyle{definition}
	\newtheorem{definition}[theorem]{Definition}

\theoremstyle{remark}

\theoremstyle{plain}
    \newtheorem{innercustomgeneric}{\customgenericname}
    \providecommand{\customgenericname}{}
    \newcommand{\newcustomtheorem}[2]{%
      \newenvironment{#1}[1]
      {%
       \renewcommand\customgenericname{#2}%
       \renewcommand\theinnercustomgeneric{##1}%
       \innercustomgeneric
      }
      {\endinnercustomgeneric}
    }
    
    \newcustomtheorem{customthm}{Theorem}
    \newcustomtheorem{customlemma}{Lemma}
    \newcustomtheorem{customprop}{Proposition}
\theoremstyle{plain}
\usepackage[english]{babel}
\addto\captionsenglish{
  \renewcommand{\contentsname}%
    {Supplementary materials for \\ \enquote{Deconditional Downscaling with Gaussian Processes}}%
}

\title{Deconditional Downscaling with Gaussian Processes}

\author{%
    Siu Lun Chau\thanks{Indicates equal contribution}    \strut    \thanks{Department of Statistics, Oxford, UK, OX1 3LB. <siu.chau@stats.ox.ac.uk, shahine.bouabid@stats.ox.ac.uk, dino.sejdinovic@stats.ox.ac.uk>}\\ 
    University of Oxford \\ 
    \And
    Shahine Bouabid\footnotemark[1]  \strut    \footnotemark[2] \\
    University of Oxford \\ 
    \And
    Dino Sejdinovic\footnotemark[2] \\
    University of Oxford\\
}

\begin{document}

\maketitle

\begin{abstract}
    Refining low-resolution (LR) spatial fields with high-resolution (HR) information, often known as \emph{statistical downscaling}, is challenging as the diversity of spatial datasets often prevents direct matching of observations. Yet, when LR samples are modeled as aggregate conditional means of HR samples with respect to a mediating variable that is globally observed, the recovery of the underlying fine-grained field can be framed as taking an \enquote{inverse} of the conditional expectation, namely a \emph{deconditioning problem}. In this work, we propose a Bayesian formulation of deconditioning which naturally recovers the initial reproducing kernel Hilbert space formulation from \citet{hsu19a}. We extend deconditioning to a downscaling setup and devise efficient conditional mean embedding estimator for multiresolution data. By treating conditional expectations as inter-domain features of the underlying field, a posterior for the latent field can be established as a solution to the deconditioning problem. Furthermore, we show that this solution can be viewed as a two-staged vector-valued kernel ridge regressor and show that it has a minimax optimal convergence rate under mild assumptions. Lastly, we demonstrate its proficiency in a synthetic and a real-world atmospheric field downscaling problem, showing substantial improvements over existing methods.
\end{abstract}

\addtocontents{toc}{\protect\setcounter{tocdepth}{-1}} 
\section{Introduction}
\label{section:introduction}

Spatial observations often operate at limited resolution due to practical constraints. For example, remote sensing atmosphere products~\cite{TheMODISAerosolAlgorithmProductsandValidation, modis_cloud, Stephens, 700993} provide measurement of atmospheric properties such as cloud top temperatures and optical thickness, but only at a low resolution. Devising methods to refine low-resolution (LR) variables for local-scale analysis thus plays a crucial part in our understanding of the anthropogenic impact on climate

When high-resolution (HR) observations of different covariates are available, details can be instilled into the LR field for refinement. This task is referred to as \emph{statistical downscaling} or \emph{spatial disaggregation} and models LR observations as the aggregation of an unobserved underlying HR field. For example, multispectral optical satellite  imagery~\cite{malenovsky, landsat} typically comes at higher resolution than atmospheric products and can be used to refine the latter.

Statistical downscaling has been studied in various forms, notably giving it a probabilistic treatment~\cite{zhang2020aggregate, LeonLaw2018, Hamelijnck2019a, Yousefi2019, Tanaka2019a, VilleTanskanenKristaLongi2020}, in which Gaussian processes (GP)~\cite{rasmussen2005gaussian} are typically used in conjunction with a sparse variational formulation~\cite{Titsias2009VariationalLO} to recover the underlying unobserved HR field. Our approach follows this line of research where we do not observe data from the underlying HR groundtruth field. On the other hand, deep neural network (DNN) based approaches~\cite{vaughan2021convolutional,groenke2020climalign,vandal2017deepsd} study this problem from a different setting, where they often assume that both HR and LR matched observations are available for training. Then, their approaches follow a standard supervised learning setting in learning a mapping between different resolutions.

However, both lines of existing methods require access to bags of HR covariates that are paired with aggregated targets, which in practice might be infeasible. For example, the multitude of satellites in orbit not only collect snapshots of the atmosphere at different resolutions, but also from different places and at different times, such that these observations are not jointly observed. To overcome this limitation, we propose to consider a more flexible \emph{mediated statistical downscaling} setup that only requires indirect matching between LR and HR covariates through a mediating field. We assume that this additional field can be easily observed, and matched separately with both HR covariates and aggregate LR targets. We then use this third-party field to mediate learning and downscale our unmatched data. In our motivating application, climate simulations~\cite{eyring2016overview, flato2011esm, scholze2012esm} based on physical science can serve as a mediating field since they provide a comprehensive spatiotemporal coverage of meteorological variables that can be matched to both LR and HR covariates.

Formally, let $^b\!\boldx = \{x^{(1)}, \ldots, x^{(n)}\}\subset\cX$ be a general notation for bags of HR covariates, $f:\cX\rightarrow\RR$ the field of interest we wish to recover and $\tilde z$ the LR aggregate observations from the field $f$. We suppose that $^b\!\boldx$ and $\tilde z$ are unmatched, but that there exists mediating covariates $y, \tilde y\in\cY$, such that $(^b\!\boldx, y)$ are jointly observed and likewise for $(\tilde y, \tilde z)$ as illustrated in Figure~\ref{fig: illustration}. We assume the following aggregation observation model $\tilde z = \EE_X[f(X)|Y=\tilde y] + \varepsilon$ with some noise $\varepsilon$. Our goal in mediated statistical downscaling is then to estimate $f$ given $(^b\!\boldx, y)$ and $(\tilde y, \tilde z)$, which corresponds to the \emph{deconditioning} problem introduced in \cite{hsu19a}.

\begin{wrapfigure}{rt}{0.35\textwidth}
    \centering
    \includegraphics[width=0.35\textwidth]{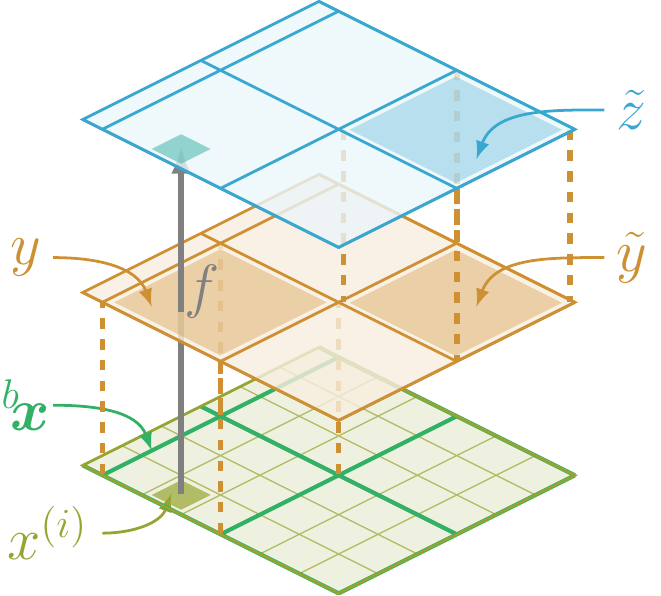}
    \caption{The LR response $\tilde z$ (blue) and the bag HR covariates $^b\!\boldx$ (green) are unmatched. The downscaling is mediated through bag-level LR covariates $y$ and $\tilde y$ (orange).}
    \label{fig: illustration}
\end{wrapfigure}

Motivated by applications in likelihood-free inference and task-transfer regression, \citet{hsu19a} first studied the deconditioning problem through the lens of reproducing kernel Hilbert space (RKHS) and introduced the framework of \emph{Deconditional Mean Embeddings} (DME) as its solution.

In this work, we first propose a Bayesian formulation of deconditioning that results into a much simpler and elegant way to arrive to the DME-based estimator of \citet{hsu19a}, using the conditional expectations of $f$. Motivated by our mediated statistical downscaling problem, we then extend deconditioning to a multi-resolution setup and bridge the gap between DMEs and existing probabilistic statistical downscaling methods~\cite{LeonLaw2018}. By placing a GP prior on the sought field $f$, we obtain a posterior distribution of the downscaled field as a principled Bayesian solution to the downscaling task on indirectly matched data. For scalability, we provide a tractable variational inference approximation and an alternative approximation to the conditional mean operator (CMO)~\cite{muandet2016kernel} to speed up computation for large multi-resolution datasets.

From a theoretical stand point, we further develop the framework of DMEs by establishing it as a two-staged vector-valued regressor with a natural reconstruction loss that mirrors \citet{grunewalder2012conditional}'s work on conditional mean embeddings. This perspective allows us to leverage distribution regression theory from \cite{Szabo2016, singh2019kernel} and obtain convergence rate results for the deconditional mean operator (DMO) estimator. Under mild assumptions, we obtain conditions under which this rate is a minimax optimal in terms of statistical-computational efficiency.

Our contributions are summarized as follows:
\begin{itemize}
    \item We propose a Bayesian formulation of the deconditioning problem of \citet{hsu19a}. We establish its posterior mean as a DME-based estimate and its posterior covariance as a gauge of the deconditioning quality.
    \item We extend deconditioning to a multi-resolution setup in the context of the mediated statistical downscaling problem and bridge the gap with existing probabilistic statistical downscaling methods. Computationally efficient algorithms are devised.
    \item We demonstrate that the DMO estimate minimises a two-staged vector-valued regression and derive its convergence rate under mild assumptions, with conditions for minimax optimality.
    \item We benchmark our model against existing methods for statistical downscaling tasks in climate science, on both synthetic and real-world multi-resolution atmospheric fields data, and show improved performance.
\end{itemize}
\section{Background Materials}
\label{section:background}

\subsection{Notations}

Let $X, Y$ be a pair of random variables taking values in non-empty sets $\cX$ and $\cY$, respectively. Let $k: \cX \times \cX \rightarrow \RR$ and $\ell: \cY \times \cY \rightarrow \RR$ be positive definite kernels. The closure of the span of their canonical feature maps $k_x:=k(x, \cdot)$ and $\ell_y:=\ell(y, \cdot)$ for $x\in\cX$ and $y\in\cY$ respectively induces RKHS $\cH_{k}\subseteq \RR^\cX$ and $\cH_{\ell} \subseteq \RR^\cY$ endowed with inner products $\langle\cdot, \cdot\rangle_k$ and $\langle\cdot, \cdot\rangle_\ell$.

We observe realizations $^b\!\cD_1 = \{^b\!\boldx_j, y_j\}_{j=1}^N$ of bags $^b\!\boldx_j = \{x_j^{(i)}\}_{i=1}^{n_j}$ from conditional distribution $\PP_{X|Y=y_j}$, with bag-level covariates $y_j$ sampled from $\PP_Y$. We concatenate them into vectors $\bfx := \begin{bmatrix}^b\!\boldx_1& \hdots & ^b\!\boldx_N\end{bmatrix}^\top$ and $\bfy := \begin{bmatrix}y_1 & \hdots & y_N\end{bmatrix}^\top$.

For simplicity, our exposition will use the notation without bagging -- i.e.\ where
$\cD_1 = \{x_j, y_j\}_{j=1}^N$ -- when the generality of our contributions will be relevant to the theory of deconditioning from \citet{hsu19a}, in Sections~\ref{section:background}, \ref{section:deconditioning-with-gps} and \ref{section:deconditioning-as-regression}. We will come back to a bagged dataset formalism in Sections~\ref{subsection:shrinkage-cmo} and \ref{section:experiments}, which corresponds to our motivating application of mediated statistical downscaling.

With an abuse of notation, we define feature matrices by stacking feature maps along columns as $\bfPhi_\bfx := \begin{bmatrix}k_{x_1} & \hdots & k_{x_N}\end{bmatrix}$ and $\bfPsi_\bfy := \begin{bmatrix}\ell_{y_1} & \hdots & \ell_{y_N}\end{bmatrix}$ and we denote Gram matrices as $\bfK_{\bfx\bfx} = \bfPhi_\bfx^\top\bfPhi_\bfx = [k(x_i, x_j)]_{1\leq i, j\leq N}$ and $\bfL_{\bfy\bfy} = \bfPsi_\bfy^\top\bfPsi_\bfy = [\ell(y_i, y_j)]_{1\leq i, j\leq N}$. The notation abuse $(\cdot)^\top(\cdot)$ is a shorthand for the elementwise RKHS inner products when it is clear from the context.

Let $Z$ denote the real-valued random variable stemming from the noisy conditional expectation of some unknown latent function $f:\cX\rightarrow\RR$, as $Z = \EE_X[f(X)|Y] + \varepsilon$. We suppose one observes another set of realizations $\cD_2 = \{\tilde y_j, \tilde z_j\}_{j=1}^M$ from $\PP_{YZ}$, which is sampled independently from $\cD_1$. Likewise, we stack observations into vectors $\tilde\bfy := \begin{bmatrix}\tilde y_1 & \hdots & \tilde y_M\end{bmatrix}^\top$, $\tilde \bfz := \begin{bmatrix}\tilde z_1 & \hdots & \tilde z_M\end{bmatrix}^\top$ and define feature map $\bfPsi_{\tilde\bfy} := \begin{bmatrix}\ell_{\tilde y_1} & \hdots & \ell_{\tilde y_M}\end{bmatrix}$.

\subsection{Conditional and Deconditional Kernel Mean Embeddings}
\paragraph*{Marginal and Joint Mean Embeddings} Kernel mean embeddings of distributions provide a powerful framework for representing and manipulating distributions without specifying their parametric form~\cite{song2013kernel, muandet2016kernel}. The marginal mean embedding of measure $\PP_X$ is defined as $\mu_{X} := \EE_X[k_X]\in\cH_k$ and corresponds to the Riesz representer of expectation functional $f\mapsto \EE_X[f(X)]$. It can hence be used to evaluate expectations $\EE_X[f(X)] = \langle f, \mu_X\rangle_k$. If the mapping $\PP_X\mapsto\mu_X$ is injective, the kernel $k$ is said to be characteristic~\cite{fukumizu2004dimensionality}, a property satisfied for the Gaussian and Matérn kernels on $\RR^d$~\cite{fukumizu2008kernel}. In practice, Monte Carlo estimator $\hat{\mu}_{X}:=\frac{1}{N}\sum_{i=1}^N k_{x_i}$ provides an unbiased estimate of $\mu_X$~\cite{sriperumbudur2012empirical}.

Extending this rationale to embeddings of joint distributions, we define $C_{YY} := \EE_{Y}[\ell_Y\otimes\ell_Y]\in\cH_\ell\otimes\cH_\ell$ and $C_{XY} := \EE_{X,Y}[k_X\otimes\ell_Y]\in\cH_k\otimes\cH_\ell$, which can be identified with the cross-covariance operators between Hilbert spaces $C_{YY}:\cH_\ell\rightarrow\cH_\ell$ and $C_{XY}:\cH_\ell\rightarrow\cH_k$. They correspond to the Riesz representers of bilinear forms $(g, g')\mapsto\operatorname{Cov}(g(Y), g'(Y)) = \langle g, C_{YY}g'\rangle_\ell$ and $(f, g)\mapsto\operatorname{Cov}(f(X), g(Y)) = \langle f, C_{XY}g\rangle_k$. As above, empirical estimates are obtained as $\hat C_{YY} = \frac{1}{N}\bfPsi_\bfy\bfPsi_\bfy^\top = \frac{1}{N}\sum_{i=1}^N\ell_{y_i}\otimes\ell_{y_i}$ and $\hat C_{XY} = \frac{1}{N}\bfPhi_\bfx\bfPsi_\bfy^\top = \frac{1}{N}\sum_{i=1}^N k_{x_i}\otimes\ell_{y_i}$. Again, notation abuse $(\cdot)(\cdot)^\top$ is a shorthand for element-wise tensor products when clear from context.

\paragraph*{Conditional Mean Embeddings} Similarly, one can introduce RKHS embeddings for conditional distributions referred to as \emph{Conditional Mean Embeddings} (CME). The CME of conditional probability measure $\PP_{X|Y=y}$ is defined as $\mu_{X|Y=y} := \EE_X[k_X|Y=y] \in \cH_k$. As introduced by \citet{fukumizu2004dimensionality}, it is common to formulate conditioning in terms of a Hilbert space operator $C_{X|Y}: \cH_{\ell} \rightarrow \cH_k$ called the \textit{Conditional Mean Operator} (CMO). $C_{X|Y}$ satisfies by definition $C_{X|Y}\ell_y = \mu_{X|Y=y}$ and $C_{X|Y}^\top f = \EE_X[f(X)|Y=\cdot]\,, \forall f\in\cH_k$, where $C_{X|Y}^\top$ denotes the adjoint of $C_{X|Y}$. Plus, the CMO admits expression $C_{X|Y} = C_{XY} C_{YY}^{-1}$, provided $\ell_y \in \Range{C_{YY}},\,\forall y\in \cY$~\cite{fukumizu2004dimensionality, song2013kernel}. \citet{song2009hilbert} show that a nonparametric empirical form of the CMO writes
\begin{equation}\label{eq:empirical-cmo}
    \hat{C}_{X|Y} = \bfPhi_\bfx(\bfL_{\bfy\bfy} + N\lambda {\bf I}_N)^{-1}\bfPsi_\bfy^\top,
\end{equation}
where $\lambda > 0$ is some regularisation ensuring the empirical operator is globally defined and bounded.

As observed by \citet{grunewalder2012conditional}, since $C_{X|Y}$ defines a mapping from $\cH_\ell$ to $\cH_k$, it can be interpreted as the solution to a vector-valued regression problem. This perspective enables derivation of probabilistic convergence bounds on the empirical CMO estimator~\cite{grunewalder2012conditional, singh2019kernel}.

\paragraph*{Deconditional Mean Embeddings} Introduced by \citet{hsu19a} as a new class of embeddings, \textit{Deconditional Mean Embeddings} (DME) are natural counterparts of CMEs. While CME $\mu_{X|Y=y}\in\cH_k$ allows to take the conditional expectation of any $f\in\cH_k$ through inner product $\EE_X[f(X)|Y=y] = \langle f, \mu_{X|Y=y}\rangle_k$, the DME denoted $\mu_{X=x|Y}\in\cH_\ell$ solves the inverse problem\footnote{the slightly unusual notation $\mu_{X=x|Y}$ is taken from \citet{hsu19a} and is meant to contrast the usual conditioning $\mu_{X|Y=y}$} and allows to recover the initial function of which the conditional expectation was taken, through inner product $\langle \EE_X[f(X)|Y=\cdot], \mu_{X=x|Y}\rangle_\ell = f(x)$.

The associated Hilbert space operator, the \textit{Deconditional Mean Operator} (DMO), is thus defined as the operator $D_{X|Y}:\cH_k\rightarrow\cH_\ell$ such that $D_{X|Y}^\top\EE_X[f(X)|Y=\cdot] = f\,,\forall f\in\cH_k$. It admits an expression in terms of CMO and cross-covariance operators $D_{X|Y} = (C_{X|Y}C_{YY})^\top (C_{X|Y}C_{YY}(C_{X|Y})^\top)^{-1}$ provided $\ell_y\in\Range{C_{YY}}$ and $k_x\in\Range{C_{X|Y}C_{YY}C_{X|Y}^\top}\,,\forall y\in\cY$ and $\forall x\in\cX$. A nonparametric empirical estimate of the DMO using datasets $\cD_1$ and $\cD_2$ as described above, is given by $\hat D_{X|Y} = \bfPsi_{\tilde\bfy}\left({\bf A}^\top \bfK_{\bfx\bfx} {\bf A} + M\epsilon {\bf I}_M\right)^{-1}{\bf A}^\top \bfPhi_\bfx^\top$ where $\epsilon > 0$ is a regularisation term and ${\bf A}:=(\bfL_{\bfy\bfy} + N\lambda {\bf I})^{-1}\bfL_{\bfy\tilde\bfy}$ can be interpreted as a mediation operator. Applying the DMO to expected responses $\tilde\bfz$, \citet{hsu19a} are able to recover an estimate of $f$ as 
\begin{equation}\label{eq:empirical-dmo-applied}
\hat f(x) = k(x, \bfx){\bf A}\left({\bf A}^\top \bfK_{\bfx\bfx} {\bf A} + M\epsilon {\bf I}_M\right)^{-1}\tilde\bfz.
\end{equation}
Note that since separate samples $\tilde{\bfy}$ can be used to estimate $C_{YY}$, this naturally fits a mediating variables setup where $\bfx$ and the conditional means $\tilde\bfz$ are not jointly observed.

\section{Deconditioning with Gaussian processes}
\label{section:deconditioning-with-gps}

In this section, we introduce \textit{Conditional Mean Process} (CMP), a stochastic process stemming from the conditional expectation of a GP. We provide a characterisation of the CMP and show that the corresponding posterior mean of its integrand is a DME-based estimator. We also derive in Appendix~\ref{appendix:variational} a variational formulation of our model that scales to large datasets and demonstrate its performance in Section~\ref{section:experiments}.

For simplicity, we put aside observations bagging in Sections~\ref{subsection:cmp} and \ref{subsection:deconditional-posterior}, our contributions being relevant to the general theory of DMEs. We return to a bagged formalism in Section~\ref{subsection:shrinkage-cmo} and extend deconditioning to the multiresolution setup inherent to the mediated statistical downscaling application. In what follows, $\cX$ is a measurable space, $\cY$ a Borel space and feature maps $k_x$ and $\ell_y$ are Borel-measurable functions for any $x\in\cX$, $y\in\cY$. All proofs and derivations of the paper are included in the appendix.

\subsection{Conditional Mean Process}\label{subsection:cmp}
Bayesian quadrature~\cite{larkin1972gaussian, briol2019probabilistic, rasmussen2005gaussian} is based on the observation that the integral of a GP with respect to some marginal measure is a Gaussian random variable. We start by probing the implications of integrating with respect to conditional distribution $\PP_{X|Y=y}$ and considering such integrals as functions of the conditioning variable $y$. This gives rise to the notion of conditional mean processes.

\begin{definition}[Conditional Mean Process]
Let $f \sim \cG\cP (m, k)$ with integrable sample paths, i.e.\ $\int_\cX |f|\d\PP_X < \infty$ a.s. The CMP induced by $f$ with respect to \ $\PP_{X|Y}$ is defined as the stochastic process $\left\{g(y)\,:\,y\in\mathcal Y\right\}$ given by 
    \begin{equation}
        g(y) = \EE_X[f(X)|Y=y] = \int_\cX f(x) \d\PP_{X|Y=y}(x).
    \end{equation}
\end{definition}
By linearity of the integral, it is clear that $g(y)$ is a Gaussian random variable for each $y\in\mathcal Y$. The sample paths integrability requirement ensures $g$ is well-defined almost everywhere. The following result characterizes CMP as a GP on $\cY$.

\begin{proposition}[CMP characterization]\label{proposition:cmp-characterization} Suppose $\EE_X[|m(X)|] < \infty$ and $\EE_X[\|k_X\|_k] < \infty$ and let $(X', Y')\sim\PP_{XY}$. Then $g$ is a Gaussian process $g\sim\cG\cP(\nu, q)$ a.s.\ , specified by
\begin{equation}
    \nu(y) = \EE_X[m(X)|Y=y] \qquad\quad     q(y, y')  = \EE_{X, X'}[k(X, X')|Y=y, Y'=y']\label{eq:cmp-covar-function}
\end{equation}
$\forall y, y'\in\cY$. Furthermore,  $q(y, y') = \langle \mu_{X|Y=y}, \mu_{X|Y=y'}\rangle_k$ a.s.\
\end{proposition}

Intuitively, the CMP can be understood as a GP on the conditional means where its covariance $q(y, y')$ is induced by the similarity between the CMEs at $y$ and $y'$. Resorting to the kernel $\ell$ defined on $\cY$, we can reexpress the covariance using Hilbert space operators as $q(y, y') = \langle C_{X|Y}\ell_y, C_{X|Y}\ell_{y'}\rangle_k$. A natural nonparametric estimate of the CMP covariance thus comes using the CMO estimator from (\ref{eq:empirical-cmo}) as $\hat q(y, y') = \ell(y, \bfy)\left(\bfL_{\bfy\bfy} + N\lambda {\bf I}_N\right)^{-1}\bfK_{\bfx\bfx} \left(\bfL_{\bfy\bfy} + N\lambda {\bf I}_N\right)^{-1}\ell(\bfy, y')$. When $m\in\cH_k$, the mean function can be written as $\nu(y) = \langle \mu_{X|Y=y}, m\rangle_k$ for which we can also use empirical estimate $\hat \nu(y) = \ell(y, \bfy)\left(\bfL_{\bfy\bfy} + N\lambda {\bf I}_N\right)^{-1}\bfPhi_\bfx^\top m$. Finally, one can also quantify the covariance between the CMP $g$ and its integrand $f$, i.e.\ $\operatorname{Cov}(f(x), g(y)) = \EE_X[k(x, X)|Y=y]$. Under the same assumptions as Proposition~\ref{proposition:cmp-characterization}, this covariance can be expressed using mean embeddings, i.e.\ $\operatorname{Cov}(f(x), g(y)) = \langle k_x, \mu_{X|Y=y}\rangle_k$ and admits empirical estimate $k(x, \bfx)\left(\bfL_{\bfy\bfy} + N\lambda {\bf I}_N\right)^{-1}\ell(\bfy, y)$.

\subsection{Deconditional Posterior}\label{subsection:deconditional-posterior}
Given independent observations introduced above, $\cD_1 = \{\bfx, \bfy\}$ and $\cD_2 = \{\tilde\bfy, \tilde\bfz\}$, we may now consider an additive noise model with CMP prior on aggregate observations $\tilde\bfz|\tilde\bfy\sim\cN(g(\tilde\bfy), \sigma^2{\bf I}_M)$. Let $\bfQ_{\tilde\bfy\tilde\bfy} := q(\tilde\bfy, \tilde\bfy)$ be the kernel matrix induced by $q$ on $\tilde\bfy$ and let $\boldupsilon := \operatorname{Cov}(f(\bfx), \tilde\bfz) = \bfPhi_\bfx^\top C_{X|Y}\bfPsi_{\tilde\bfy}$ be the cross-covariance between $f(\bfx)$ and $\tilde\bfz$. The joint normality of $f(\bfx)$ and $\tilde\bfz$ gives
\begin{equation}\label{eq:joint-gp-and-cmp}
    \begin{bmatrix}
    f(\bfx) \\
    \tilde\bfz
    \end{bmatrix}\mid {\bfy, \tilde\bfy}\sim\cN\left(\begin{bmatrix}
    m(\bfx)\\
    \nu(\tilde\bfy)
    \end{bmatrix}, \begin{bmatrix}
     \bfK_{\bfx\bfx} & \boldupsilon \\
    \boldupsilon^\top & \bfQ_{\tilde\bfy\tilde\bfy} + \sigma^2{\bf I}_M
    \end{bmatrix}\right).
\end{equation}
Using Gaussian conditioning, we can then readily derive the posterior distribution of the underlying GP field $f$ given the  aggregate observations $\tilde\bfz$ corresponding to $\tilde\bfy$. The latter can naturally be degenerated if observations are paired, i.e. $\bfy=\tilde\bfy$. This formulation can be seen as an example of the inter-domain GP~\cite{rudner2020inter}, where we utilise the observed conditional means $\tilde\bfz$ as inter-domain inducing features for inference of $f$.

\begin{proposition}[Deconditional Posterior]\label{proposition:deconditional-posterior}
    Given aggregate observations $\tilde\bfz$ with homoscedastic noise $\sigma^2$, the deconditional posterior of $f$ is defined as the Gaussian process $f|\tilde\bfz\sim\cG\cP(m_\rmd, k_\rmd)$ where
    \begin{align}
        m_\rmd(x) & = m(x) + k_x^\top C_{X|Y}\bfPsi_{\tilde\bfy} (\bfQ_{\tilde\bfy\tilde\bfy} + \sigma^2{\bf I}_M)^{-1}(\tilde\bfz - \nu(\tilde\bfy)),\\
        k_\rmd(x, x') & = k(x, x') - k_x^\top C_{X|Y}\bfPsi_{\tilde\bfy} (\bfQ_{\tilde\bfy\tilde\bfy} + \sigma^2{\bf I}_M)^{-1}\bfPsi_{\tilde\bfy}^\top C_{X|Y}^\top k_{x'}.
    \end{align}
\end{proposition}
Substituting terms by their empirical forms, we can define a nonparametric estimate of the $m_\rmd$ as
    \begin{equation}\label{eq:deconditional-posterior-mean}
        \hat{m}_\rmd(x)  := m(x) + k(x, \bfx){\bf A}(\hat\bfQ_{\tilde\bfy\tilde\bfy} + \sigma^2{\bf I}_M)^{-1}(\tilde\bfz - \hat\nu(\tilde\bfy)))
    \end{equation}
which, when $m=0$, reduces to the DME-based estimator in (\ref{eq:empirical-dmo-applied}) by taking the noise variance $\frac{\sigma^2}{N}$ as the inverse regularization parameter. \citet{hsu19a} recover a similar posterior mean expression in their Bayesian interpretation of DME. However, they do not link the distributions of $f$ and its CMP, which leads to much more complicated chained inference derivations combining fully Bayesian and MAP estimates, while we naturally recover it using simple Gaussian conditioning.

Likewise, an empirical estimate of the deconditional covariance is given by
    \begin{equation}\label{eq:deconditional-posterior-covariance}
        \hat k_\rmd(x, x') := k(x, x') - k(x, \bfx){\bf A}(\hat\bfQ_{\tilde\bfy\tilde\bfy} + \sigma^2{\bf I}_M)^{-1}{\bf A}^\top k(\bfx, x').
    \end{equation}
Interestingly, the latter can be rearranged to write as the difference between the original kernel and the kernel undergoing conditioning and deconditioning steps $\hat k_\rmd(x, x') = k(x, x') - \langle k_x, \hat D_{X|Y}\hat C_{X|Y} k_{x'}\rangle_k$. This can be interpreted as a measure of reconstruction quality, which degenerates in the case of perfect deconditioning, i.e.\ $\hat D_{X|Y}\hat C_{X|Y} = \operatorname{Id}_{\cH_k}$.

\subsection{Deconditioning and multiresolution data}\label{subsection:shrinkage-cmo}

Downscaling application would typically correspond to multiresolution data, with bag dataset $^b\!\cD_1 = \{(^b\!\boldx_j, y_j)\}_{j=1}^N$ where $^b\!\boldx_j = \{x_j^{(i)}\}_{i=1}^{n_j}$. In this setup, the mismatch in size between vector concatenations $\bfx= [x_1^{(1)} \; \hdots \; x_N^{(n_N)}]^\top$ and $\bfy = \begin{bmatrix} y_1 & \hdots & y_N\end{bmatrix}^\top$ prevents from readily applying (\ref{eq:empirical-cmo}) to estimate the CMO and thus infer the deconditional posterior. There is, however, a straightforward approach to alleviate this: simply replicate bag-level covariates $y_j$ to match bags sizes $n_j$. Although simple, this method incurs a $\cO((\sum_{j=1}^N n_j)^3)$ cost due to matrix inversion in (\ref{eq:empirical-cmo}). 

Alternatively, since bags $^b\!\boldx_j$ are sampled from conditional distribution $\PP_{X|Y=y_j}$, unbiased Monte Carlo estimators of CMEs are given by $\hat\mu_{X|Y=y_j} = \frac{1}{n_j}\sum_{i=1}^{n_j}k_{x_j^{(i)}}$. Let ${\bf \hat M}_{\bfy} = [\hat\mu_{X|Y=y_1} \hdots \hat\mu_{X|Y=y_N}]^\top$ denote their concatenation along columns. We can then rewrite the cross-covariance operator as $C_{XY} = \EE_Y[\EE_{X}[k_X|Y]\otimes\ell_Y]$ and hence take $\frac{1}{N}{\bf \hat M}_{\bfy}\bfPsi_\bfy^\top$ as an estimate for $C_{XY}$. Substituting empirical forms into $C_{X|Y} = C_{X|Y}C_{YY}^{-1}$ and applying Woodbury identity, we obtain an alternative CMO estimator that only requires inversion of a $N \times N$ matrix. We call it \emph{Conditional Mean Shrinkage Operator} and define it as 
\begin{equation}
    ^S\hat{C}_{X|Y} := {\bf\hat M}_\bfy (\bfL_{\bfy\bfy} + \lambda N{\bf I}_N)^{-1}\bfPsi_\bfy^\top.
\end{equation} This estimator can be seen as a generalisation of the Kernel Mean Shrinkage Estimator~\cite{Muandet2016shrinkage} to the conditional case. We provide in Appendix~\ref{appendix:shrinkage-estimator} modifications of (\ref{eq:deconditional-posterior-mean}) and (\ref{eq:deconditional-posterior-covariance}) including this estimator for the computation of the deconditional posterior.
\section{Deconditioning as regression}
\label{section:deconditioning-as-regression}

In Section~\ref{subsection:deconditional-posterior}, we obtain a DMO-based estimate for the posterior mean of $f|\tilde\bfz$. When the estimate gets closer to the exact operator, the uncertainty collapses and the Bayesian view meets the frequentist. It is however unclear how the empirical operators effectively converge in finite data size. Adopting an alternative perspective, we now demonstrate that the DMO estimate can be obtained as the minimiser of a two-staged vector-valued regression. This frequentist turn enables us to leverage rich theory of vector-valued regression and establish under mild assumptions a convergence rate on the DMO estimator, with conditions to fulfill minimax optimality in terms of statistical-computational efficiency. In the following, we briefly review CMO's vector-valued regression viewpoint and construct an analogous regression problem for DMO. We refer the reader to \cite{paulsen2016introduction} for a comprehensive overview of vector-valued RKHS theory. As for Sections~\ref{subsection:cmp} and \ref{subsection:deconditional-posterior}, this section contributes to the general theory of DMEs and we hence put aside the bag notations.

\paragraph*{Stage 1: Regressing the Conditional Mean Operator}\label{subsection:stage1}
As first introduced by \citet{grunewalder2012conditional} and generalised to infinite dimensional spaces by \citet{singh2019kernel}, estimating $C_{X|Y}^\top$ is equivalent to solving a vector-valued kernel ridge regression problem in the hypothesis space of Hilbert-Schmidt operators from $\cH_k$ to $\cH_\ell$, denoted as  $\operatorname{HS}(\cH_k, \cH_\ell)$. Specifically, we may consider the operator-valued kernel defined over $\cH_k$ as $\Gamma(f, f') := \langle f, f\rangle_k \operatorname{Id}_{\cH_\ell}$. We denote $\cH_\Gamma$ the $\cH_\ell$-valued RKHS spanned by $\Gamma$ with norm $\|\cdot\|_\Gamma$, which can be identified to $\operatorname{HS}(\cH_k, \cH_\ell)$\footnote{$ \cH_\Gamma = \cSpan{\Gamma_fh, f\in\cH_k, h\in\cH_\ell} = \cSpan{f\otimes h, f\in\cH_k, h\in\cH_\ell} = \overline{\cH_k\otimes\cH_\ell}\cong\operatorname{HS}(\cH_k, \cH_\ell)$}. \citet{singh2019kernel} frame CMO regression as the minimisation surrogate discrepancy $\cE_\rmc(C) :=\EE_{X,Y}\left[\|k_X - C^\top\ell_Y\|^2_k\right]$, to which they substitute an empirical regularised version restricted to $\cH_\Gamma$ given by 
\begin{equation}\label{eq:empirical-cmo-regression-objective}
    \hat \cE_\rmc(C) :=\frac{1}{N}\sum_{i=1}^N\|k_{x_i} - C^\top\ell_{ y_i}\|^2_k + \lambda\|C\|^2_\Gamma\,\qquad C\in\cH_\Gamma\qquad\lambda > 0
\end{equation}
This $\cH_k$-valued kernel ridge regression problem admits a closed-form minimiser which shares the same empirical form as the CMO, i.e.\ $\hat C_{X|Y}^\top=\arg\min_{C\in\cH_\Gamma}\,\hat\cE_\rmc(C)$~\cite{grunewalder2012conditional, singh2019kernel}.

\paragraph*{Stage 2 : Regressing the Deconditional Mean Operator}\label{subsection:stage2}
The DMO on the other hand is defined as the operator $D_{X|Y}:\cH_k\rightarrow\cH_\ell$ such that $\forall f\in\cH_k$, $D_{X|Y}^\top C_{X|Y}^\top f = f$. Since deconditioning corresponds to finding a pseudo-inverse to the CMO, it is natural to consider a reconstruction objective $\cE_\rmd(D):= \EE_Y\left[\|\ell_Y - D C_{X|Y}\ell_Y\|^2_\ell\right]$. Introducing a novel characterization of the DMO, we propose to minimise this objective in the hypothesis space of Hilbert-Schmidt operators $\operatorname{HS}(\cH_k, \cH_\ell)$ which identifies to $\cH_\Gamma$. As per above, we denote $\hat{C}_{X|Y}$ the empirical CMO learnt in Stage 1, and we substitute the loss with an empirical regularised formulation on $\cH_\Gamma$
\begin{equation}\label{eq:empirical-dmo-as-regression}
    \hat\cE_\rmd(D) := \frac{1}{M}\sum_{j=1}^M\|\ell_{\tilde y_j} - D\hat{C}_{X|Y}\ell_{\tilde y_j}\|^2_{\ell} + \epsilon\|D\|^2_\Gamma\qquad D\in\cH_\Gamma\qquad\epsilon >0
\end{equation}

\begin{proposition}[Empirical DMO as vector-valued regressor]\label{proposition:dmo-as-regression}
    The minimiser of the empirical reconstruction risk is the empirical DMO, i.e. $\hat{D}_{X|Y} = \arg\min_{D\in\cH_{\Gamma}} {\hat\cE_\rmd(D)}$
\end{proposition}
Since it requires to estimate the CMO first, minimising (\ref{eq:empirical-dmo-as-regression}) can be viewed as a two-staged vector value regression problem.

\paragraph*{Convergence results}
Following the footsteps of \cite{singh2019kernel, Szabo2016}, this perspective enables us to state the performance of the DMO estimate in terms of asymptotic convergence of the objective $\cE_\rmd$. As in \citet{caponnetto2007optimal}, we must restrict the class of probability measure for $\PP_{XY}$ and $\PP_Y$ to ensure uniform convergence even when $\cH_k$ is infinite dimensional. The family of distribution considered is a general class of priors that does not assume parametric distributions and is parametrized by two variables: $b > 1$ controls the effective input dimension and $c\in]1, 2]$ controls functional smoothness. Mild regularity assumptions are also placed on the original spaces $\cX, \cY$, their corresponding RKHS $\cH_k, \cH_\ell$ and the vector-valued RKHS $\cH_\Gamma$. We discuss these assumptions in details in Appendix~\ref{appendix: details on convergence result}. Importantly, while $\cH_k$ can be infinite dimensional, we nonetheless have to assume the RKHS $\cH_\ell$ is finite dimensional. In further research, we will relax this assumption.

\begin{theorem}[Empirical DMO Convergence Rate]\label{theorem:convergence-rate} Denote $D_{\PP_Y} = \arg\min_{D\in\cH_\Gamma}\cE_\rmd(D)$. Assume assumptions stated in Appendix~\ref{appendix: details on convergence result} are satisfied. In particular, let $(b, c)$ and $(0, c')$ be the parameters of the restricted class of distribution for $\PP_{Y}$ and $\PP_{XY}$ respectively and let $\iota\in]0, 1]$ be the Hölder continuity exponent in $\cH_\Gamma$. Then, if we choose $\lambda = N^{-\frac{1}{c' + 1}}$, $N = M^{\frac{a(c' + 1)}{\iota(c' - 1)}}$ where $a > 0$, we have the following result,
\begin{itemize}
    \item If $a\leq\frac{b(c+1)}{bc+1}$, then $\cE_\rmd(\hat D_{X|Y}) - \cE_\rmd(D_{\PP_Y}) = \cO(M^\frac{-ac}{c+1})$  with $\epsilon=M^\frac{-a}{c+1}$
    \item If $a\geq\frac{b(c+1)}{bc+1}$, then $\cE_\rmd(\hat D_{X|Y}) - \cE_\rmd(D_{\PP_Y}) = \cO(M^\frac{-bc}{bc+1})$  with $\epsilon=M^\frac{-b}{bc+1}$
\end{itemize}
\end{theorem}

This theorem underlines a trade-off between the computational and statistical efficiency with respect to the datasets cardinalities $N = |\cD_1|$, $M = |\cD_2|$ and the problem difficulty $b, c, c'$. For $a \leq \frac{b(c+1)}{bc+1}$, smaller $a$ means less samples from $\cD_1$ at fixed $M$ and thus computational savings. But it also hampers convergence, resulting in reduced statistical efficiency. At $a=\frac{b(c+1)}{bc+1} < 2$, convergence rate is a minimax computational-statistical efficiency optimal, i.e.\ convergence rate is optimal with smallest possible $M$. We note that at this optimal, $N > M$ and which means less samples are required from $\cD_2$. $a\geq\frac{b(c+1)}{bc+1}$ does not improve the convergence rate but only increases the size of $\cD_1$ and hence the computational cost it bears.

\section{Deconditional Downscaling Experiments}
\label{section:experiments}

We demonstrate and evaluate our CMP-based downscaling approaches on both synthetic experiments and a challenging atmospheric temperature field downscaling problem with unmatched multi-resolution data. We denote the exact CMP deconditional posterior from Section~\ref{section:deconditioning-with-gps} as \textsc{cmp}, the \textsc{cmp} using with efficient shrinkage CMO estimation as \textsc{s-cmp} and the variational formulation as \textsc{varcmp}. They are compared against \textsc{vbagg}~\cite{LeonLaw2018} --- which we describe below --- and a GP regression~\cite{rasmussen2005gaussian} baseline (\textsc{gpr}) modified to take bags centroids as the input. Experiments are implemented in \emph{PyTorch}~\cite{pytorch, gardner2018gpytorch}, all code and datasets are made available\footnote{\url{https://github.com/shahineb/deconditional-downscaling}} and computational details are provided in Appendix~\ref{appendix:section:experiments}.

\paragraph*{Variational Aggregate Learning} \textsc{vbagg} is introduced by \citet{LeonLaw2018} as a variational aggregate learning framework to disaggregate exponential family models, with emphasis on the Poisson family. We consider its application to the Gaussian family, which models the relationship between aggregate targets $z_j$ and bag covariates $\{x_j^{(i)}\}_i$ by bag-wise averaging of a GP prior on the function of interest. In fact, the Gaussian \textsc{vbagg} corresponds exactly to a special case of \textsc{cmp} on matched data, where the bag covariates are simply one hot encoded indices with kernel $\ell(j, j') = \delta(j, j')$ where $\delta$ is the Kronecker delta. However, \textsc{vbagg} cannot handle unmatched data as bag indices do not instill the smoothness that is used for mediation. For fair analysis, we compare variational methods \textsc{varcmp} and \textsc{vbagg} together, and exact methods \textsc{cmp}/\textsc{s-cmp} to an exact version of \textsc{vbagg}, which we implement and refer to as \textsc{bagg-gp}.

\subsection{Swiss Roll}
The \emph{scikit-learn}~\cite{scikit-learn} swiss roll manifold sampling function allows to generate a 3D manifold of points $x\in\RR^3$ mapped with their position along the manifold $t\in\RR$. Our objective will be to recover $t$ for each point $x$ by only observing $t$ at an aggregate level. In the first experiment, we compare our model to existing weakly supervised spatial disaggregation methods when all high-resolution covariates are matched with a coarse-level aggregate target. We then proceed to withdraw this requirement in a companion experiment.

\begin{wrapfigure}{r}{0.41\textwidth}
    \centering
    \includegraphics[width=0.41\textwidth]{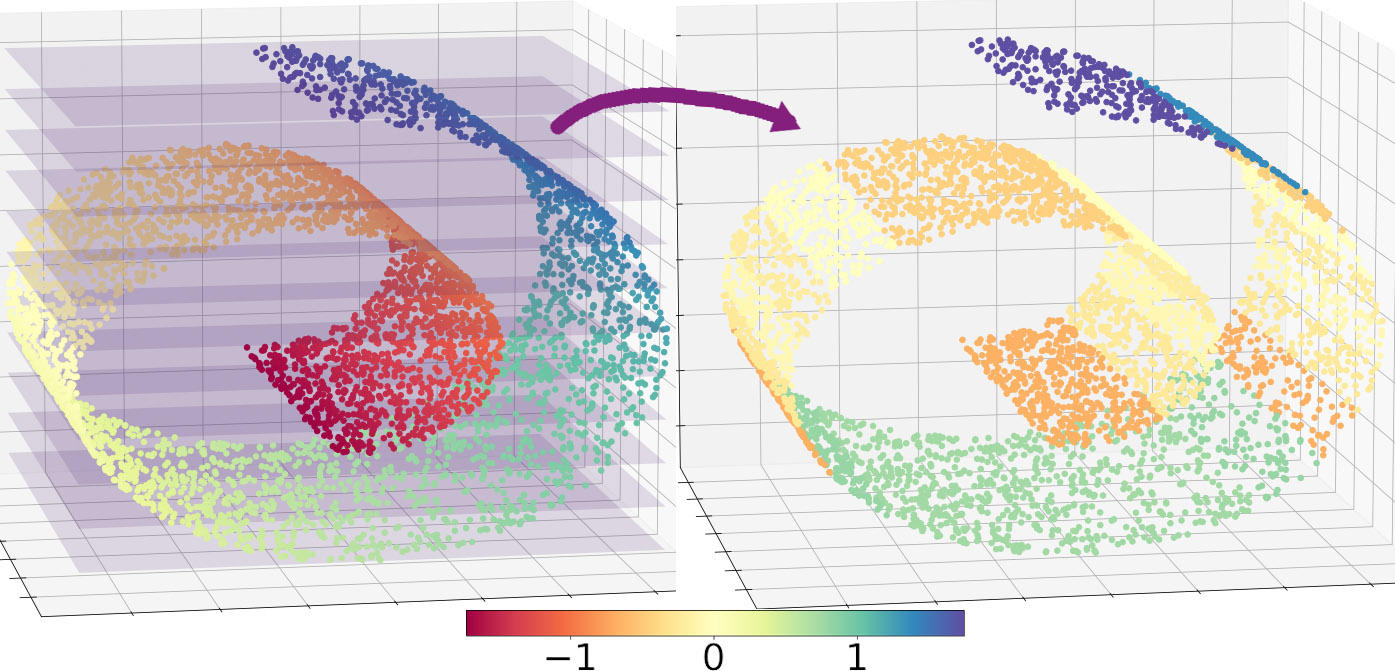}
    \caption{{\bf Step 1:} Split space regularly along height. {\bf Step 2:} Group points into height-level bags. {\bf Step 3:} Average points targets into bag-level aggregate targets.}
    \label{fig:swiss-roll}
\end{wrapfigure}

\subsubsection{Direct matching}
\paragraph*{Experimental Setup} Inspired by the experimental setup from \citet{LeonLaw2018}, we regularly split space along height $B-1$ times as depicted in Figure~\ref{fig:swiss-roll} and group together manifold points within each height level, hence mixing together points with very different positions on the manifold. We obtain bags of samples $\{(^b\!\boldx_j, \boldt_j)\}_{j=1}^B$ where the $j$\textsuperscript{th} bag contains $n_j$ points $^b\!\boldx_j = \{x_j^{(i)}\}_{i=1}^{n_j}$ and their corresponding targets ${\boldt}_j = \{t_j^{(i)}\}_{i=1}^{n_j}$. We then construct bag aggregate targets by taking noisy bag targets average $z_j := \frac{1}{n_j}\sum_{i=1}^{n_j}t_j^{(i)} + \varepsilon_j$, where $\varepsilon_j\sim\cN(0, \sigma^2)$. We thus obtain matched weakly supervised bag learning dataset $\cD^\circ=\{(^b\!\boldx_j, z_j)\}_{j=1}^B$. Since each bag corresponds to a height-level, the center altitude of each height split $y_j\in\RR$ is a natural candidate bag-level covariate that informs on relative positions of the bags. We can augment the above dataset as $\cD = \{(^b\!\boldx_j, y_j, z_j)\}_{j=1}^B$. Using these bag datasets, we now wish to downscale aggregate targets $z_j$ to recover the unobserved manifold locations $\{\boldt_j\}_{j=1}^B$ and be able to query the target at any previously unseen input $x$.

\paragraph*{Models} We use a zero-mean prior on $f$ and choose a Gaussian kernel for $k$ and $\ell$. Inducing points location is initialized with K-means++ procedure for \textsc{varcmp} and \textsc{vbagg} such that they spread evenly across the manifold. For exact methods, kernel hyperparameters and noise variance $\sigma^2$ are learnt on $\cD$ by optimising the marginal likelihood. For \textsc{varcmp}, they are learnt jointly with variational distribution parameters by maximising an evidence lower bound objective. While CMP-based methods can leverage bag-level covariates $y_j$, baselines are restricted to learn from $\cD^\circ$. Adam optimiser~\cite{Adam} is used in all experiments.

\paragraph*{Results}We test models against unobserved groundtruth $\{\boldt_j\}_{j=1}^B$ by evaluating the root mean square error (RMSE) to the posterior mean. Table~\ref{table:swiss-roll-results} shows that \textsc{cmp}, \textsc{s-cmp} and \textsc{varcmp} outperform their corresponding counterparts i.e.\ \textsc{bagg-gp} and \textsc{vbagg}, with statistical significance confirmed by a Wilcoxon signed-rank test in Appendix~\ref{appendix:section:experiments}. Most notably, this shows that the additional knowledge on bag-level dependence instilled by $\ell$ is reflected even in a setting where each bag is matched with an aggregate target.

\begin{table}[t]
\centering
\caption{RMSE of the swissroll experiment for models trained over directly and indirectly matched datasets ; scores averaged over 20 seeds and 1 s.d is reported ; * indicates our proposed methods.}
\label{table:swiss-roll-results}
\begin{tabular}{ccccccc}\toprule
Matching & \textsc{cmp}* & \textsc{s-cmp}* & \textsc{varcmp}* & \textsc{bagg-gp} & \textsc{vbagg} & \textsc{gpr} \\ \midrule
Direct  & 0.33\tiny{$\pm$0.06}   & 0.25 \tiny{$\pm$0.04}     & {\bf 0.18}\tiny{$\pm$0.04}      & 0.60\tiny{$\pm$0.01}      & 0.22\tiny{$\pm$0.04}     & 0.70\tiny{$\pm$0.05}   \\ 
Indirect & {\bf 0.80}\tiny{$\pm$0.14} & 1.05\tiny{$\pm$0.04} & 0.87\tiny{$\pm$0.07} & 1.13\tiny{$\pm$0.11} & 1.46\tiny{$\pm$0.34} & 1.04\tiny{$\pm$0.05}\\
\bottomrule
\end{tabular}
\vspace{-0.2cm}
\end{table}

\subsubsection{Indirect matching}
\paragraph*{Experimental Setup}We now impose indirect matching through mediating variable $y_j$. We randomly select $N = \lfloor\frac{B}{2}\rfloor$ bags which we consider to be the $N$ first ones without loss of generality and split $\cD$ into $\cD_1 = \{(^b\!\boldx_j, y_j)\}_{j=1}^{N}$ and $\cD_2 = \{(\tilde y_j, \tilde z_j)\}_{j=1}^{B - N} = \{(y_{N + j}, z_{N + j})\}_{j=1}^{B-N}$, such that no pair of covariates bag $^b\!\boldx_j$ and aggregate target $\tilde z_j$ are jointly observed.

\paragraph*{Models}CMP-based methods are naturally able to learn from this setting and are trained by independently drawing samples from $\cD_1$ and $\cD_2$. Baseline methods however require bags of covariates to be matched with an aggregate bag target. To remedy this, we place a separate prior $g\sim\cG\cP(0, \ell)$ and fit regression model $\tilde z_j = g(\tilde y_j) + \varepsilon_j$ over $\cD_2$. We then use the predictive posterior mean to augment the first dataset as $\cD_1' = \left\{\left(^b\!\boldx_j, \EE[g(y_j)|\cD_2]\right)\right\}_{j=1}^{N}$. This dataset can then be used to train \textsc{bagg-gp}, \textsc{vbagg} and \textsc{gpr}.

\paragraph*{Results}For comparison, we use the same evaluation as in the direct matching experiment. Table~\ref{table:swiss-roll-results} underlines an anticipated drop in RMSE for all models, but we observe that \textsc{bagg-gp} and \textsc{vbagg} suffer most from the mediated matching of the dataset while \textsc{cmp} and \textsc{varcmp} report best scores by a substantial margin. This highlights how using a separate prior on $g$ to mediate $\cD_1$ and $\cD_2$ turns out to be suboptimal in contrast to using the prior naturally implied by CMP. While it is more computationally efficient than \textsc{cmp}, we observe a relative drop in performance for \textsc{s-cmp}.

\subsection{Mediated downscaling of atmospheric temperature}
Given the large diversity of sources and formats of remote sensing and model data, expecting directly matched observations is often unrealistic~\cite{cis}. For example, two distinct satellite products will often provide low and high resolution imagery that can be matched neither spatially nor temporally~\cite{TheMODISAerosolAlgorithmProductsandValidation, modis_cloud, Stephens, 700993}. Climate simulations~\cite{flato2011esm, scholze2012esm, eyring2016overview} on the other hand provide a comprehensive coarse resolution coverage of meteorological variables that can serve as a mediating dataset.

For the purpose of demonstration, we create an experimental setup inspired by this problem using Coupled Model Intercomparison Project Phase 6 (CMIP6)~\cite{eyring2016overview} simulation data. This grants us access to groundtruth high-resolution covariates to facilitate model evaluation.

\paragraph*{Experimental Setup}We collect monthly mean 2D atmospheric fields simulation from CMIP6 data~\cite{cmip6_rest, cmip6_ttop} for the following variables: air temperature at cloud top (T), mean cloud top pressure (P), total cloud cover (TCC) and cloud albedo ($\alpha$). First, we collocate TCC and $\alpha$ onto a HR latitude-longitude grid of size 360$\times$720 to obtain fine grained fields (latitude\textsuperscript{HR}, longitude\textsuperscript{HR}, altitude\textsuperscript{HR}, TCC\textsuperscript{HR}, $\alpha$\textsuperscript{HR}) augmented with a static HR surface altitude field. Then we collocate P and T onto a LR grid of size 21$\times$42 to obtain coarse grained fields (latitude\textsuperscript{LR}, longitude\textsuperscript{LR}, P\textsuperscript{LR}, T\textsuperscript{LR}). We denote by $B$ the number of low resolution pixels.

Our objective is to disaggregate T\textsuperscript{LR} to the HR fields granularity.
We assimilate the $j$\textsuperscript{th} coarse temperature pixel to an aggregate target $z_j:= $ T$_j^{\text{LR}}$ corresponding to bag $j$. Each bag includes HR covariates $^b\!\boldx_j = \{x_j^{(i)}\}_{i=1}^{n_j} := \{($latitude$_{j}^{\text{HR}(i)}$, longitude$_{j}^{\text{HR}(i)}$, altitude$_{j}^{\text{HR}(i)}$, TCC$_{j}^{\text{HR}(i)}$, $\alpha_{j}^{\text{HR}(i)})\}_{i=1}^{n_j}$. To emulate unmatched observations, we randomly select $N = \lfloor\frac{B}{2}\rfloor$ of the bags $\{^b\!\boldx_j\}_{j=1}^{N}$ and keep the opposite half of LR observations $\{z_{N + j}\}_{j=1}^{B-N}$, such that there is no single aggregate bag target that corresponds to one of the available bags. Finally, we choose the pressure field P\textsuperscript{LR} as the mediating variable. We hence compose a third party low resolution field of bag-level covariates $y_j :=($latitude$_j^\text{LR}$, longitude$_j^\text{LR}$, P$_j^\text{LR})$ which can separately be matched with both above sets to obtain datasets $\cD_1 = \{(\boldx_j, y_j)\}_{j=1}^{N}$ and $\cD_2 = \{(\tilde y_{j}, \tilde z_{j})\}_{j=1}^{B - N} = \{(y_{N + j}, z_{N + j})\}_{j=1}^{B - N}$.

\begin{figure}[t]
    \centering
    \caption{{\bf Left:} High-resolution atmoshperic covariates used for prediction; {\bf Center-Left:} Observed low-resolution temperature field, grey pixels are unobserved; {\bf Center} Unobserved high-resolution groundtruth temperature field; {\bf Center-Right:} \textsc{varcmp} deconditional posterior mean; {\bf Right} 95\% confidence region size on prediction; temperature values are in Kelvin.}
    \label{fig:downscaling-posterior}
    \includegraphics[width=\linewidth]{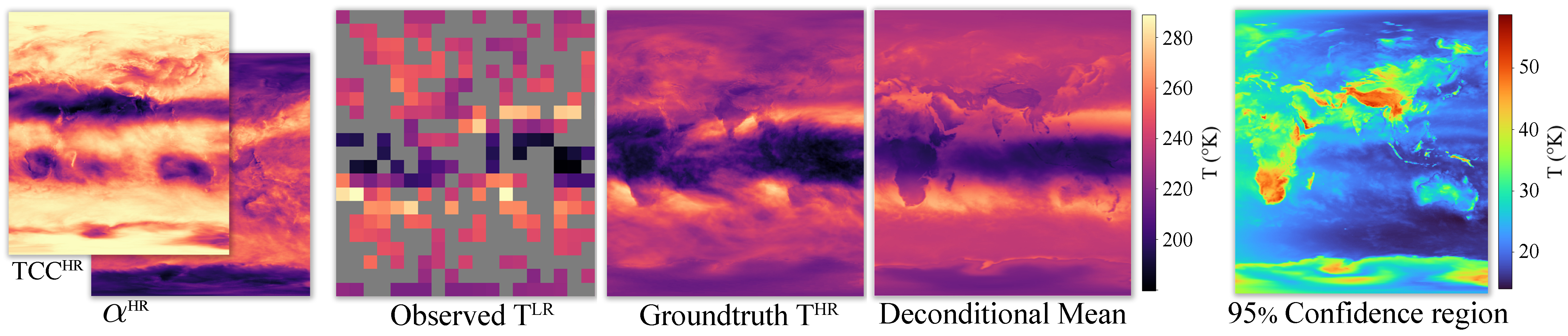}
\end{figure}

\paragraph*{Models Setup}We only consider variational methods to scale to large number of pixels. \textsc{varcmp} is naturally able to learn from indirectly matched data. We use a Mat\'{e}rn-1.5 kernel for rough spatial covariates (latitude, longitude) and a Gaussian kernel for atmospheric covariates (P, TCC, $\alpha$) and surface altitude. $k$ and $\ell$ are both taken as sums of Mat\'{e}rn and Gaussian kernels, and their hyperparameters are learnt along with noise variance during training. A high-resolution noise term is also introduced, with details provided in Appendix~\ref{appendix:section:experiments}. Inducing points locations are uniformly initialized across the HR grid. We replace \textsc{gpr} with an inducing point variational counterpart \textsc{vargpr}~\cite{Titsias2009VariationalLO}. Since baseline methods require a matched dataset, we proceed as with the unmatched swiss roll experiment and fit a GP regression model $g$ with kernel $\ell$ on $\cD_2$ and then use its predictive posterior mean to obtain pseudo-targets for the bags of HR covariates from $\cD_1$.

\paragraph*{Results}Performance is evaluated by comparing downscaling deconditional posterior mean against original high resolution field T\textsuperscript{HR} available in CMIP6 data~\cite{cmip6_ttop}, which we emphasise is never observed. We use random Fourier features~\cite{rahimi2007random} approximation of kernel $k$ to scale kernel evaluation to the HR covariates grid during testing. As reported in Table~\ref{table:downscaling-scores}, \textsc{varcmp} substantially outperforms both baselines with statistical significance provided in Appendix~\ref{appendix:section:experiments}. Figure~\ref{fig:downscaling-posterior} shows the reconstructed image with \textsc{varcmp}, plots for other methods are included in the Appendix~\ref{appendix:section:experiments}. The model resolves statistical patterns from HR covariates into coarse resolution temperature pixels, henceforth reconstructing a faithful HR version of the temperature field.

\begin{table}[t]
\centering
\caption{Downscaling similarity scores of posterior mean against groundtruth high resolution cloud top temperature field ; averaged over 10 seeds; we report 1 s.d. ; \enquote{\small{$\downarrow$}}: lower is better ; \enquote{\small{$\uparrow$}}: higher is better.}
\begin{tabular}{rcccc}\toprule
Model &  RMSE\small{$\;\downarrow$}  & MAE\small{$\;\downarrow$}  & Corr.\small{$\;\uparrow$} & SSIM\small{$\;\uparrow$}  \\\midrule
\textsc{vargpr} & 8.02\scriptsize{$\pm$0.28}  & 5.55\scriptsize{$\pm$0.17} & 0.831\scriptsize{$\pm$0.012} & {\bf 0.212}\scriptsize{$\pm$0.011} \\
\textsc{vbagg}    & 8.25\scriptsize{$\pm$0.15} & 5.82\scriptsize{$\pm$0.11} & 0.821\scriptsize{$\pm$0.006} & 0.182\scriptsize{$\pm$0.004} \\
\textsc{varcmp} & {\bf 7.40}\scriptsize{$\pm$0.25}  & {\bf 5.34}\scriptsize{$\pm$0.22} & {\bf 0.848}\scriptsize{$\pm$0.011} & {\bf 0.212}\scriptsize{$\pm$0.013} \\\bottomrule
\end{tabular}
\label{table:downscaling-scores}
\vspace{-0.2cm}
\end{table}

\section{Discussion}
We introduced a scalable Bayesian solution to the mediated statistical downscaling problem, which handles unmatched multi-resolution data. The proposed approach combines Gaussian Processes with the framework of deconditioning using RKHSs and recovers previous approaches as its special cases. We provided convergence rates for the associated deconditioning operator. Finally, we demonstrated the advantages over spatial disaggregation baselines in synthetic data and in a challenging atmospheric fields downscaling problem.

In future work, exploring theoretical guarantees of the computationally efficient shrinkage formulation in a multi-resolution setting and relaxing finite dimensionality assumptions for the convergence rate will have fruitful practical and theoretical implications. Further directions also open up to quantify uncertainty over the deconditional posterior since it is computed using empirical estimates of the CMP covariance. This may be problematic if the mediating variable undergoes covariate shift between the two datasets.
\section*{Acknowledgments}
SLC is supported by the EPSRC and MRC through the OxWaSP CDT programme EP/L016710/1. SB is supported by the EU's Horizon 2020 research and innovation programme under Marie Skłodowska-Curie grant agreement No 860100. DS is supported in partly by Tencent AI lab and partly by the Alan Turing Institute (EP/N510129/1).



\bibliographystyle{unsrtnat}
\bibliography{references}

\newpage
\appendix
\tableofcontents
\addtocontents{toc}{\protect\setcounter{tocdepth}{3}} 
\newpage
\section{Proofs}

\subsection{Proofs of Section~\ref{section:deconditioning-with-gps}}

\begin{proposition}\label{proposition:from-expectation-to-conditional}
    Let $h : \cX\rightarrow\RR$ such that $\EE[|h(X)|]<\infty$. Then, $\{y\in\cY \mid \EE[|h(X)| | Y=y] < \infty\}$ is a full measure set with respect to $\PP_Y$.
\end{proposition}
\begin{proof}
Since $\cX$ is a Borel space and $\cY$ is measurable, the existence of a $\PP_Y$-a.e. regular conditional probability distribution is guaranteed by \cite[Theorem 6.3]{kallenberg2002foundations}.
Now suppose $\EE[|h(X)|] < \infty$ and let $\cY^o = \{y\in\cY\mid \EE[|h(X)||Y=y] < \infty\}$. Since $\EE[|h(X)|] = \EE\left[\EE[|h(X)|\mid Y]\right]$, the conditional expectation $\EE[|h(X)|\mid Y]$ must have finite expectation almost everywhere, i.e.\ $\PP_Y(\cY^o) = 1$. 
\end{proof}

\begin{customprop}{3.2}
    Suppose $\EE[|m(X)|] < \infty$ and $\EE[\|k_X\|_k] < \infty$ and let $(X', Y')\sim\PP_{XY}$. Then $g$ is a Gaussian process $g\sim\cG\cP(\nu, q)$ a.s.\ , specified by
\begin{equation}
    \nu(y) = \EE[m(X)|Y=y] \qquad\quad     q(y, y')  = \EE[k(X, X')|Y=y, Y'=y']\label{eq:cmp-covar-function}
\end{equation}
$\forall y, y'\in\cY$. Furthermore, $q(y, y') = \langle \mu_{X|Y=y}, \mu_{X|Y=y'}\rangle_k$ a.s.
\end{customprop}
\begin{proof}[Proof of Proposition~\ref{proposition:cmp-characterization}] 

We will assume for the sake of simplicity that $m = 0$ in the following derivations and will return to the case of an uncentered GP at the end of the proof.

\paragraph*{Show that $g(y)$ is in a space of Gaussian random variables}

Let $(\Omega, \cF, \PP)$ denote a probability space and $L^2(\Omega, \PP)$ the space of square integrable random variables endowed with standard inner product. $\forall x\in\cX$, since $f(x)$ is Gaussian, then $f(x)\in L^2(\Omega, \PP)$. We can hence define $\cS(f)$ as the closure in $L^2(\Omega, \PP)$ of the vector space spanned by $f$, i.e. $\cS(f) := \cSpan{f(x)\, : \, x\in\cX}$.

Elements of $\cS(f)$ write as limits of centered Gaussian random variables, hence when their covariance sequence converge, they are normally distributed. Let $T\in\cS(f)^\perp$, then we have $\EE[T f(x)] = 0$. Let $y\in\cY$, we also have
\begin{align}
    \EE[Tg(y)] = \EE\left[\int_\cX T f(x)\d\PP_{X|Y=y}\right]
\end{align}

In order to switch orders of integration, we need to show that the double integral satisfies absolute convergence.
\begin{align}
    \int_\cX \EE[|Tf(x)|]\d\PP_{X|Y=y}(x) & \leq \int_\cX \sqrt{\EE[T^2]\EE[f(x)^2]}\d\PP_{X|Y=y}(x) \\
    & = \sqrt{\EE[T^2]}\int_\cX \|k_x\|_k\d\PP_{X|Y=y}(x) \\
    & = \sqrt{\EE[T^2]}\EE[\|k_X\|_k|Y=y]
\end{align}

Since $T\in L^2(\Omega, \PP)$, $\EE[T^2] < \infty$. Plus, as we assume that $\EE[\|k_X\|_k] < \infty$, Proposition~\ref{proposition:from-expectation-to-conditional} gives that $\EE[\|k_X\|_k|Y=y] < \infty$ a.s. We can thus apply Fubini's theorem and obtain
\begin{align}
    \EE[Tg(y)] = \int_\cX \EE[Tf(x)]\d\PP_{X|Y=y}(x) = 0 \enspace\text{a.s.}
\end{align}

As this holds for any $T\in\cS(f)^\perp$, we conclude that $g(y)\in\left(\cS(f)^\perp\right)^\perp \text{ a.s.}\Rightarrow g(y)\in\cS(f) \text{ a.s.}$. We cannot claim yet though that $g(y)$ is Gaussian since we do not know whether it results from a sequence of Gaussian variables with converging variance sequence. We now have to prove that $g(y)$ has a finite variance.

\paragraph*{Show that $g(y)$ has finite variance}

We proceed by computing the expression of the covariance between $g(y)$ and $g(y')$ which is more general and yields the variance.

Let $y, y'\in\cY$, the covariance of $g(y)$ and $g(y')$ is given by
\begin{align}
    q(y, y') & = \EE[g(y)g(y')] - \EE[g(y)]\EE[g(y')] \\
    & = \EE\left[\int_\cX\int_\cX f(x)f(x')\d\PP_{X|Y=y}(x)\d\PP_{X|Y=y'}(x')\right] \label{eq:proof-covar-terms}\\
    & - \EE\left[\int_\cX f(x)\d\PP_{X|Y=y}(x)\right]\EE\left[\int_\cX f(x')\d\PP_{X|Y=y'}(x')\right] \label{eq:proof-mean-terms}
\end{align}

Choosing $T$ as a constant random variable in the above, we can show that $\int_\cX \EE[|f(x)|]\d\PP_{X|Y=y}(x) < \infty$ a.s. We can hence apply Fubini's theorem to switch integration order in the mean terms (\ref{eq:proof-mean-terms}) and obtain that $\EE[g(y)] = 0$ since $f$ is centered.

To apply Fubini's theorem to (\ref{eq:proof-covar-terms}), we need to show that the triple integration absolutely converges. Let $x, x'\in\cX$, we know that $\EE[|f(x)f(x')|] \leq \sqrt{\EE[f(x)^2]\EE[f(x')^2]} = \|k_x\|_k\|k_{x'}\|_k$. Using similar arguments as above, we obtain
\begin{align}
\int_\cX\int_\cX \EE[|f(x)f(x')|]\d\PP_{X|Y=y}(x)\d\PP_{X|Y=y'}(x') \leq \EE[\|k_X\|_k|Y=y]\EE[\|k_X\|_k |Y=y'] < \infty\text{ a.s.}
\end{align}

We can thus apply Fubini's theorem which yields
\begin{align}
    q(y, y') & = \int_\cX\int_\cX \EE[f(x)f(x')]\d\PP_{X|Y=y}(x)\d\PP_{X|Y=y'}(x') \\
    & = \int_\cX\int_\cX \underbrace{\operatorname{Cov}(f(x), f(x'))}_{k(x, x')}\d\PP_{X|Y=y}(x)\d\PP_{X|Y=y'}(x') \\
    & = \EE[k(X, X')| Y=y, Y'=y'] \\
    & \leq \EE[\|k_X\|_k|Y=y]\EE[\|k_X\|_k|Y=y'] < \infty\enspace \text{a.s.}
\end{align}
where $(X', Y')$ denote random variables with same joint distribution than $(X, Y)$ as defined in the proposition.

$g(y)\in\cS(f)$ and has finite variance $q(y, y)$ a.s., it is thus a centered Gaussian random variable a.s. Furthermore, as this holds for any $y\in\cY$, then any finite subset of $\{g(y)\,:\, y\in\cY\}$ follows a multivariate normal distribution which shows that $g$ is a centered Gaussian process on $\cY$ and its covariance function is specified by $q$.

\paragraph*{Uncentered case $m\neq 0$}

We now return to an uncentered GP prior on $f$ with assumption that $\EE[|m(X)|] < \infty$. By Proposition~\ref{proposition:from-expectation-to-conditional}, we get that $\EE[|m(X)|\,|Y=y] < \infty$ a.s. for $y\in\cY$.

Let $\nu : y\mapsto \EE[m(X)|Y=y]$. We can clearly rewrite $g$ as the sum of $\nu$ and a centered GP on $\cY$
\begin{equation}
    g(y) = \nu(y) + \int_\cX (f(x) - m(x))\d\PP_{X|Y=y}(x),\qquad\forall y\in\cY
\end{equation}
which is well-defined almost surely.

It hence comes $\EE[g(y)] = \EE[\nu(y)] + 0 = \nu(y)$. Plus since $\nu(y)$ is a constant shift, the covariance is not affected and has the same expression than for the centered GP. Since this holds for any $y\in\cY$, we conclude that $g\sim\cG\cP(\nu, q)$ a.s.

\paragraph*{Show that $q(y, y') = \langle\mu_{X|Y=y}, \mu_{X|Y=y'}\rangle_k$}

First, we know by Proposition~\ref{proposition:from-expectation-to-conditional} that $\EE[\|k_X\|_k|Y=y] < \infty$ $\PP_Y$-a.e.\ .By triangular inequality, we obtain $\|\mu_{X|Y=y}\|_k = \|\EE[k_X|Y=y]\|_k \leq \EE[\|k_X\|_k|Y=y] < \infty$ $\PP_Y$-a.e. and hence $\mu_{X|Y=y}$ is well-defined up to a set of measure zero with respect to $\PP_Y$.

With notations from Proposition~\ref{proposition:cmp-characterization}, we can proceed for any $y, y'\in\cY$ as
\begin{align}
    q(y, y') & = \EE[k(X, X')|Y=y', Y'=y'] \\
             & = \int_\cX\int_\cX k(x, x') \d\PP_{X|Y=y}(x)\d\PP_{X|Y=y'}(x') \\
             & = \int_\cX\int_\cX \langle k_x, k_{x'}\rangle_k \d\PP_{X|Y=y}(x)\d\PP_{X|Y=y'}(x') \\
             & = \left\langle \int_\cX k_x\d\PP_{X|Y=y}(x), \int_\cX k_{x'}\d\PP_{X|Y=y'}(x') \right\rangle_k \enspace\text{ a.s.}\\
             & = \langle \mu_{X|Y=y}, \mu_{X|Y=y'}\rangle_k \enspace\text{ a.s.}
\end{align}
\end{proof}

\begin{customprop}{3.3}
    Given aggregate observations $\tilde\bfz$ with homoscedastic noise $\sigma^2$, the deconditional posterior of $f$ is defined as the Gaussian process $f|\tilde\bfz\sim\cG\cP(m_\rmd, k_\rmd)$ where
    \vspace*{-0.1cm}
    \begin{align}
        m_\rmd(x) & = m(x) + k_x^\top C_{X|Y}\bfPsi_{\tilde\bfy} (\bfQ_{\tilde\bfy\tilde\bfy} + \sigma^2{\bf I}_M)^{-1}(\tilde\bfz - \nu(\tilde\bfy)),\\
        k_\rmd(x, x') & = k(x, x') - k_x^\top C_{X|Y}\bfPsi_{\tilde\bfy} (\bfQ_{\tilde\bfy\tilde\bfy} + \sigma^2{\bf I}_M)^{-1}\bfPsi_{\tilde\bfy}^\top C_{X|Y}^\top k_{x'}.
    \end{align}
\end{customprop}
\begin{proof}[Proof of Proposition~\ref{proposition:deconditional-posterior}]
Recall that
\begin{equation}
    \begin{bmatrix}
    f(\bfx) \\
    \tilde\bfz
    \end{bmatrix}\mid {\bfy, \tilde\bfy}\sim\cN\left(\begin{bmatrix}
    m(\bfx)\\
    \nu(\tilde\bfy)
    \end{bmatrix}, \begin{bmatrix}
     \bfK_{\bfx\bfx} & \boldupsilon \\
    \boldupsilon^\top & \bfQ_{\tilde\bfy\tilde\bfy} + \sigma^2{\bf I}_M
    \end{bmatrix}\right).
\end{equation}
where $\boldupsilon = \operatorname{Cov}(f(\bfx), \tilde\bfz) = \bfPhi_\bfx^\top C_{X|Y}\bfPsi_{\tilde\bfy}$.

Applying Gaussian conditioning, we obtain that
\begin{align}
    f(\bfx)\mid\tilde\bfz, \bfy, \tilde\bfy \sim \cN( & m(\bfx) + \boldupsilon(\bfQ_{\tilde\bfy\tilde\bfy} + \sigma^2{\bf I}_M)^{-1}(\tilde\bfz - \nu(\tilde\bfy)), \\
    & \bfK_{\bfx\bfx} - \boldupsilon(\bfQ_{\tilde\bfy\tilde\bfy} + \sigma^2{\bf I}_M)^{-1}\boldupsilon^\top)
\end{align}

Since the latter holds for any input $\bfx\in\cX^N$, by Kolmogorov extension theorem this implies that $f$ conditioned on the data $\tilde\bfz, \tilde\bfy$ is a draw from a GP. We denote it $f|\tilde\bfz\sim\cG\cP(m_\rmd, k_\rmd)$ and it is specified by
\begin{align}
    m_\rmd(x) & = m(x) + k_x^\top C_{X|Y}\bfPsi_{\tilde\bfy} (\bfQ_{\tilde\bfy\tilde\bfy} + \sigma^2{\bf I}_M)^{-1}(\tilde\bfz - \nu(\tilde\bfy)),\\
    k_\rmd(x, x') & = k(x, x') - k_x^\top C_{X|Y}\bfPsi_{\tilde\bfy} (\bfQ_{\tilde\bfy\tilde\bfy} + \sigma^2{\bf I}_M)^{-1}\bfPsi_{\tilde\bfy}^\top C_{X|Y}^\top k_{x'}.
\end{align}

Note that we abuse notation
\begin{align}
    "k_x^\top C_{X|Y}\bfPsi_{\tilde\bfy}" & = \begin{bmatrix}\langle k_x, C_{X|Y}\ell_{\tilde y_{1}}\rangle_k & \hdots & \langle k_x, C_{X|Y}\ell_{\tilde y_{M}}\rangle_k\end{bmatrix} \\
    & =  \begin{bmatrix}\langle k_x, \mu_{X|Y=\tilde y_{1}}\rangle_k & \hdots & \langle k_x, \mu_{X|Y=\tilde y_{M}}\rangle_k\end{bmatrix} \\
    & = \begin{bmatrix}\operatorname{Cov}(f(x), g(\tilde y_1)) & \hdots & \operatorname{Cov}(f(x), g(\tilde y_M))\end{bmatrix} .
\end{align}
\end{proof}


\subsection{Proofs of Section~\ref{section:deconditioning-as-regression}}

\begin{customprop}{4.1}
[Empirical DMO as vector-valued regressor]
    The minimiser of the empirical reconstruction risk is the empirical DMO, i.e. $\hat{D}_{X|Y} = \arg\min_{D\in\cH_{\Gamma}} {\hat\cE_\rmd(D)}$    
\end{customprop}
\begin{proof}[Proof of Proposition~\ref{proposition:dmo-as-regression}]

Let $D\in\cH_\Gamma$, we recall the form of the regularised empirical objective
\begin{equation}
    \hat\cE_\rmd(D) = \frac{1}{M}\sum_{j=1}^M\|\ell_{\tilde y_j} - D\hat{C}_{X|Y}\ell_{\tilde y_j}\|^2_{\ell} + \epsilon\|D\|^2_\Gamma
\end{equation}

By \cite[Theorem 4.1]{micchelli2005learning}, if $\hat D \in \underset{D\in\cH_\Gamma}{\arg\min}\,\hat\cE_\rmd(D)$, then it is unique and has form
\begin{equation}\label{eq:solution-template}
    \hat D = \sum_{j=1}^M\Gamma_{\hat{C}_{X|Y}\ell_{\tilde y_j}}c_j
\end{equation}
where $\Gamma_{\hat{C}_{X|Y}\ell_{\tilde y_j}}: \cH_{\ell} \rightarrow \cH_{\Gamma}$ is the vector-valued kernel $\Gamma$'s feature map indexed by $\hat{C}_{X|Y}\ell_{\tilde y_j}$, such that for any $h \in \cH_{\Gamma}$ and $g \in \cH_{\ell}$, we have $\langle h, \Gamma_{\hat{C}_{X|Y}\ell_{\tilde y_j}} g \rangle_{\Gamma} = \langle h(\hat{C}_{X|Y}\ell_{\tilde y_j}), g \rangle_{\ell}$. (see \cite{paulsen2016introduction} for a detailed review of vector-valued RKHS). Furthermore, coefficients $c_1, \ldots, c_M \in\cH_\ell$ are the unique solutions to
\begin{equation}
    \sum_{i=1}^M\left(\Gamma(\hat{C}_{X|Y}\ell_{\tilde y_i}, \hat{C}_{X|Y}\ell_{\tilde y_j}) + M\epsilon\delta_{ij}\right)c_{i} = \ell_{\tilde y_j}
\end{equation}

Since 
\begin{equation}
    \Gamma(\hat{C}_{X|Y}\ell_{\tilde y_i}, \hat{C}_{X|Y}\ell_{\tilde y_j}) = \langle \hat{C}_{X|Y}\ell_{\tilde y_i}, \hat{C}_{X|Y}\ell_{\tilde y_j} \rangle_k \operatorname{Id}_{\cH_\ell} = \hat q(\tilde y_i, \tilde y_j) \operatorname{Id}_{\cH_\ell}
\end{equation}
where $\operatorname{Id}_{\cH_\ell}$ denotes the identity operator on $\cH_\ell$. The above simplifies as
\begin{align}\label{eq:coefficients-dmo-regression}
    & \sum_{i=1}^M\left(\hat q(\tilde y_i, \tilde y_j) + M\epsilon\delta_{ij}\right)c_{i} = \ell_{\tilde y_j} \quad\forall 1\leq j\leq M\\ & \Leftrightarrow  \left(\hat\bfQ_{\tilde\bfy\tilde\bfy} + M\epsilon{\bf I}_M\right){\bf c}^\top = \bfPsi_{\tilde\bfy}^\top \\ & \Leftrightarrow {\bf c} = \bfPsi_{\tilde\bfy}\left(\hat\bfQ_{\tilde\bfy\tilde\bfy} + M\epsilon{\bf I}_M\right)^{-1}
\end{align}
where ${\bf c} = \begin{bmatrix}c_1 & \hdots & c_M\end{bmatrix}$.

Since for any $f\in\cH_k$ and $g\in\cH_\ell$, our choice of kernel gives $\Gamma_f g = g\otimes f$, plugging (\ref{eq:coefficients-dmo-regression}) into (\ref{eq:solution-template}) we obtain
\begin{align}
    \hat D & = \sum_{j=1}^M \Gamma_{\hat C_{X|Y}\ell_{\tilde y_j}}c_j\\
    & = \sum_{j=1}^M c_j\otimes \hat C_{X|Y}\ell_{\tilde y_j} \\
    & = {\bf c} \left[\hat C_{X|Y}\bfPsi_{\tilde\bfy}\right]^\top \\
    & = \left[\bfPsi_{\tilde\bfy}\left(\hat\bfQ_{\tilde\bfy\tilde\bfy} + M\epsilon{\bf I}_M\right)^{-1}\right] \left[\hat C_{X|Y}\bfPsi_{\tilde\bfy}\right]^\top \\
    & = \bfPsi_{\tilde\bfy}\left(\hat\bfQ_{\tilde\bfy\tilde\bfy} + M\epsilon{\bf I}_M\right)^{-1}\bfPsi_{\tilde\bfy}^\top\hat C_{X|Y}^\top \\
    & = \bfPsi_{\tilde\bfy}\left(\hat\bfQ_{\tilde\bfy\tilde\bfy} + M\epsilon{\bf I}_M\right)^{-1}\bfPsi_{\tilde\bfy}^\top\bfPsi_{\bfy}\left(\bfL_{\bfy\bfy} + N\lambda {\bf I}_N\right)^{-1}\bfPhi_\bfx \\
    & = \bfPsi_{\tilde\bfy}\left(\hat\bfQ_{\tilde\bfy\tilde\bfy} + M\epsilon{\bf I}_M\right)^{-1}\bfA\bfPhi_\bfx \\
    & = \hat D_{X|Y}
\end{align}
which concludes the proof.
\end{proof}

\begin{customthm}{4.2}[Empirical DMO Convergence Rate]Denote $D_{\PP_Y} = \arg\min_{D\in\cH_\Gamma}\cE_\rmd(D)$. Assume assumptions stated in Appendix~\ref{appendix: details on convergence result} are satisfied. In particular, let $(b, c)$ and $(0, c')$ be the parameters of the restricted class of distribution for $\PP_{Y}$ and $\PP_{XY}$ respectively and let $\iota\in]0, 1]$ be the Hölder continuity exponent in $\cH_\Gamma$. Then, if we choose $\lambda = N^{-\frac{1}{c' + 1}}$, $N = M^{\frac{a(c' + 1)}{\iota(c' - 1)}}$ where $a > 0$, we have the following result,
\begin{itemize}
    \item If $a\leq\frac{b(c+1)}{bc+1}$, then $\cE_\rmd(\hat D_{X|Y}) - \cE_\rmd(D_{\PP_Y}) = \cO(M^\frac{-ac}{c+1})$  with $\epsilon=M^\frac{-a}{c+1}$
    \item If $a\geq\frac{b(c+1)}{bc+1}$, then $\cE_\rmd(\hat D_{X|Y}) - \cE_\rmd(D_{\PP_Y}) = \cO(M^\frac{-bc}{bc+1})$  with $\epsilon=M^\frac{-b}{bc+1}$
\end{itemize}
\end{customthm}
\begin{proof}[Proof of Theorem~\ref{theorem:convergence-rate}]
In Appendix~\ref{appendix: details on convergence result}, we present Theorem~\ref{theorem:detailed-convergence-result} which is a detailed version of this result with all assumptions explicitly stated. The proof of Theorem~\ref{theorem:detailed-convergence-result} constitutes the proof of this result.
\end{proof}

\newpage 
\section{Variational formulation of the deconditional posterior}\label{appendix:variational}

Inference computational complexity is $\cO(M^3)$ for the posterior mean and $\cO(N^3 + M^3)$ for the posterior covariance. To scale to large datasets, we introduce in the following a variational formulation as a scalable approximation to the deconditional posterior $f(\bfx)|\tilde\bfz$. Without loss of generality, we assume in the following that $f$ is centered, i.e.\ $m=0$.

\subsection{Variational formulation}
We consider a set of $d$ inducing locations $\bfw = \begin{bmatrix} w_1 & \ldots & w_d\end{bmatrix}^\top\in\cX^d$ and define inducing points as the gaussian vector ${\bf u} := f({\bf w})\sim\cN(0, {\bf K}_{\bf ww})$, where ${\bf K}_{\bf ww} := k({\bf w, w})$. We set $d$-dimensional variational distribution $q({\bf u}) = \cN({\bf u| \bfeta, \bfSigma})$ over inducing points and define $q({\bf f}) := \int p({\bf f} | {\bf u})q({\bf u})\d{\bf u}$ as an approximation of the deconditional posterior $p({\bf f} | {\bf z})$. The estimation of the deconditional posterior can thus be approximated by optimising the variational distribution parameters $\bfeta$, $\bfSigma$ to maximise the \emph{evidence lower bound} (ELBO) objective given by
\begin{equation}\label{eq:elbo}
    \operatorname{ELBO}(q) = \EE_{q({\bf f})}[\log p({\bf \tilde z| f})] + \operatorname{KL}(q({\bf u})\| p({\bf u})).
\end{equation}
As both $q$ and $p$ are Gaussians, the Kullback-Leibler divergence admits closed-form. The expected log likelihood term decomposes as
\begin{equation}\label{eq:elbo-ell}
    \EE_{q({\bf f})}[\log p({\bf z| f})] = -\frac{M}{2}\log(2\pi\sigma^2) + \frac{1}{2\sigma^2}\left(\tr\left({\bf A^\top {\bf \bar\bfSigma} A}\right)  + \left\|{\tilde\bfz - A^\top{\bf \bar{\bfeta}}}\right\|_{2}^2\right)
\end{equation}
where $\bar\bfeta$ and $\bar\bfSigma$ are the parameters of the posterior variational distribution $q({\bf f}) = \cN({\bf f}|\bar\bfeta, \bar\bfSigma)$ given by
\begin{equation}
    {\bf \bar\bfeta = K_{xw}K_{ww}^{-1}\bfeta} \qquad\quad {\bf \bar\bfSigma = K_{xx} - K_{xw}\left[K_{ww}^{-1} - K_{ww}^{-1}\bfSigma K_{ww}^{-1}\right]K_{wx}}
\end{equation}
Given this objective, we can optimise this lower bound with respect to variational parameters $\bfeta, \bfSigma$, noise $\sigma^2$ and parameters of kernels $k$ and $\ell$, with an option to parametrize these kernels using feature maps given by deep neural network~\cite{law2019hyperparameter}, using a stochastic gradient approach for example. We might also want to learn the inducing locations ${\bf w}$.

\subsection{Details on evidence lower bound derivation}
For completeness, we provide here the derivation of the evidence lower bound objective. Let us remind its expression as stated in (\ref{eq:elbo})
\begin{equation}
    \operatorname{ELBO}(q) = \EE_{q({\bf f})}[\log p({\bf\tilde z}|{\bf f})] - \operatorname{KL}(q({\bf u}) \| p({\bf u}))
\end{equation}

The second term here is the Kullback-Leibler divergence of two gaussian densities which has a known and tractable closed-form expression. 
\begin{equation}
    \operatorname{KL}(q({\bf u}) \| p({\bf u})) = \frac{1}{2}\left[\tr\left(\bfK_{\bf ww}^{-1}\bfSigma\right) + \bfeta^\top\bfK_{\bf ww}^{-1}\bfeta - d + \log\frac{\det\bfK_{\bf ww}}{\det\bfSigma} \right]
\end{equation}

The first term is the expected log likelihood and needs to be derived. Using properties of integrals of gaussian densities, we can start by showing that $q({\bf f})$ also corresponds to a gaussian density which comes
\begin{align}
    q({\bf f}) & = \int p({\bf f} | {\bf u})q({\bf u})\d{\bf u} \\
    & = \int \cN({\bf f}|{\bf K_{xw}K_{ww}^{-1}u}, {\bf K_{xx} - K_{xw}K_{ww}^{-1}K_{wx}})\times \cN({\bf u| \bfeta, \bfSigma})\d{\bf u} \\
    & = \cN({\bf f| \bar\bfeta, \bar\bfSigma})
\end{align}

where
\begin{align}
    {\bf \bar\bfeta} & = {\bf K_{xw}K_{ww}^{-1}\bfeta} \\
    {\bf \bar\bfSigma} & = {\bf K_{xx} - K_{xw}\left[K_{ww}^{-1} - K_{ww}^{-1}\bfSigma K_{ww}^{-1}\right]K_{wx}}
\end{align}

Let's try now to obtain a closed-form expression of $\EE_{q({\bf f})}[\log p({\bf\tilde z}|{\bf f})]$ on which we will be able to perform a gradient-based optimization routine. Using Gaussian conditioning on (\ref{eq:joint-gp-and-cmp}), we obtain
\begin{equation}\label{eq:63}
    p({\bf\tilde z| f}) = \cN({\bf\tilde z}| \boldsymbol{\Upsilon}^\top{\bf K}_{\bfx\bfx}^{-1}{\bf f}\;,\; {\bfQ_{\tilde\bfy\tilde\bfy}} + \sigma^2 {\bf I}_M - \boldsymbol{\Upsilon}^\top{\bf K_{\bfx\bfx}}^{-1}\boldsymbol{\Upsilon})
\end{equation}

We notice that $\boldsymbol{\Upsilon}^\top{\bf K}_{\bfx\bfx}^{-1} = \ell({\bf\tilde y,  y})({\bf L}_{\bfy\bfy} + \lambda N{\bf I}_N)^{-1}{\bf K_{\bfx\bfx} K_{\bfx\bfx}^{-1}} = \ell({\bf \tilde y,  y})({\bf L}_{\bfy\bfy} + \lambda N{\bf I}_N)^{-1} = {\bf A}$.

Hence we also have $\boldsymbol{\Upsilon}^\top{\bf K}_{\bfx\bfx}^{-1}\boldsymbol{\Upsilon} = {\bf A}^\top\bfK_{\bfx\bfx}{\bf A} = \bfQ_{\tilde\bfy\tilde\bfy}$.

We can thus simplify (\ref{eq:63}) as
\begin{equation}
    p({\bf\tilde z| f}) = \cN({\bf\tilde z}|  {\bf A^\top}{\bf f}, \sigma^2 {\bf I}_n)
\end{equation}

Then,
\begin{align}
    \log p({\bf\tilde z| f}) & = -\frac{M}{2}\log(2\pi\sigma^2)- \frac{1}{2\sigma^2}\left\|{\bf\tilde z -  A^\top f}\right\|^2_{2} \\
    \Rightarrow \EE_{q({\bf f})}[\log p({\bf\tilde z| f})] & = -\frac{M}{2}\log(2\pi\sigma^2) - \frac{1}{2\sigma^2}\EE_{q({\bf f})}\left[\left\|{\bf\tilde z -  A^\top f}\right\|^2_{2}\right]
\end{align}

Using the trace trick to express the expectation with respect to the posterior variational parameters ${\bar\bfeta, \bar\bfSigma}$, we have
\begin{align}
    \EE_{q({\bf f})}\left[\left\|{\bf\tilde z -  A^\top f}\right\|^2_{2}\right] & = \EE_{q({\bf f})}\left[\tr\left(\left({\bf\tilde z -  A^\top f}\right)^\top \left({\bf\tilde z -  A^\top f}\right)\right)\right] \\
    & = \EE_{q({\bf f})}\left[\tr\left(\left({\bf\tilde z -  A^\top f}\right)\left({\bf\tilde z -  A^\top f}\right)^\top\right)\right] \\
    & = \tr\left(\EE_{q({\bf f})}\left[\left({\bf\tilde z -  A^\top f}\right)\left({\bf\tilde z -  A^\top f}\right)^\top\right]\right) \\
\end{align}

And 
\begin{align}
    \EE_{q({\bf f})}\left[\left({\bf\tilde z -  A^\top f}\right)\left({\bf\tilde z -  A^\top f}\right)^\top\right] & = \operatorname{Cov}({\bf\tilde z -  A^\top f}) + \EE_{q({\bf f})}\left[{\bf\tilde z -  A^\top f}\right]\EE_{q({\bf f})}\left[{\bf\tilde z -  A^\top f}\right]^\top \\
    & = {\bf A^\top \bar{\bfSigma}A} + \left({\bf\tilde z - A^\top\bar{\bfeta}}\right)\left({\bf\tilde z - A^\top\bar{\bfeta}}\right)^\top
\end{align}

Hence, it comes that
\begin{align}
    \EE_{q({\bf f})}\left[\left\|{\bf\tilde z -  A^\top f}\right\|^2_{2}\right] & = \tr\left({\bf A^\top \bar{\bfSigma}A}\right) + \tr\left(\left({\bf\tilde z - A^\top\bar{\bfeta}}\right)\left({\bf\tilde z - A^\top\bar{\bfeta}}\right)^\top\right) \\
    & = \tr\left({\bf A^\top \bar{\bfSigma}A}\right)  + \left\|{\bf\tilde z - A^\top\bar{\bfeta}}\right\|_{2}^2
\end{align}
which can be efficiently computed as it only requires diagonal terms.

Wrapping up, we obtain that
\begin{equation}
    \operatorname{ELBO}(q) = -\frac{M}{2}\log(2\pi\sigma^2) - \frac{1}{2\sigma^2}\left(\tr\left({\bf A^\top \bar{\bfSigma}A}\right)  + \left\|{\bf\tilde z - A^\top\bar{\bfeta}}\right\|_{2}^2\right)- \operatorname{KL}(q({\bf u}) \| p({\bf u}))
\end{equation}

\newpage
\section{Details on Conditional Mean Shrinkage Operator}\label{appendix:shrinkage-estimator}

\subsection{Deconditional posterior with Conditional Mean Shrinkage Operator}

We recall from Proposition~\ref{proposition:deconditional-posterior} that the deconditional posterior is a GP specifed by mean and covariance functions
\begin{align}
    m_\rmd(x) & = m(x) + k_x^\top C_{X|Y}\bfPsi_{\tilde\bfy} (\bfQ_{\tilde\bfy\tilde\bfy} + \sigma^2{\bf I}_M)^{-1}(\tilde\bfz - \nu(\tilde\bfy)), \label{eq:appendix-deconditional-mean}\\
    k_\rmd(x, x') & = k(x, x') - k_x^\top C_{X|Y}\bfPsi_{\tilde\bfy}(\bfQ_{\tilde\bfy\tilde\bfy} + \sigma^2{\bf I}_M)^{-1}\bfPsi_{\tilde\bfy}^\top C_{X|Y}^\top k_{x'} \label{eq:appendix-deconditional-covar}
\end{align}
for any $x, x'\in\cX$, where we abuse notation for the cross-covariance term
\begin{equation}
   "k_x^\top C_{X|Y}\bfPsi_{\tilde\bfy}" = \begin{bmatrix} \langle k_x, C_{X|Y}\ell_{\tilde y_1} \rangle_k & \hdots & \langle k_x, C_{X|Y}\ell_{\tilde y_M} \rangle_k \end{bmatrix}.
\end{equation}

The CMO appears in the cross-covariance term $k_x^\top C_{X|Y}\bfPsi_{\tilde\bfy}$ and in the CMP covariance matrix $\bfQ_{\tilde\bfy\tilde\bfy} = \bfPsi_{\tilde\bfy}^\top C_{X|Y}^\top C_{X|Y}\bfPsi_{\tilde\bfy}$. To derive empirical versions using the Conditional Mean Shrinkage Operator we replace it by $^S\!\hat C_{X|Y} = {\bf\hat M}_{\bfy}(\bfL_{\bfy\bfy} + \lambda N{\bf I}_N)^{-1}\bfPsi_{\bfy}^\top$.

The empirical cross-covariance operator with shrinkage CMO estimate is given by
\begin{align}
    k_x^{\top} \,\!^S\!\hat C_{X|Y}\bfPsi_{\tilde\bfy} & = k_x^\top {\bf\hat M}_{\bfy}(\bfL_{\bfy\bfy} + \lambda N{\bf I}_N)^{-1}\bfPsi_{\bfy}^\top\bfPsi_{\tilde\bfy} \\
    & = k_x^\top {\bf\hat M}_{\bfy}(\bfL_{\bfy\bfy} + \lambda N{\bf I}_N)^{-1}\bfL_{\bfy\tilde\bfy} \\
    & = k_x^\top {\bf\hat M}_{\bfy}{\bf A}
\end{align}
where we abuse notation
\begin{align}
    "k_x^\top {\bf\hat M}_{\bfy}" & := \begin{bmatrix}\langle k_x, \hat\mu_{X|Y=y_1}\rangle_k & \hdots & \langle k_x, \hat\mu_{X|Y=y_N}\rangle_k\end{bmatrix} \\
    & = \begin{bmatrix} \frac{1}{n_1}\sum_{i=1}^{n_1}k(x_1^{(i)}, x) & \hdots & \frac{1}{n_N}\sum_{i=1}^{n_N}k(x_N^{(i)}, x)\end{bmatrix}.
\end{align}

The empirical shrinkage CMP covariance matrix is given by
\begin{align}
    ^S\!\hat\bfQ_{\tilde\bfy\tilde\bfy} & := \bfPsi_{\tilde\bfy}^\top \,\!^S\!\hat C_{X|Y}^\top  \,\!^S\!\hat C_{X|Y}\bfPsi_{\tilde\bfy} \\
    & = \bfPsi_{\tilde\bfy}^\top\bfPsi_{\bfy}(\bfL_{\bfy\bfy} + \lambda N{\bf I}_N)^{-1} {\bf\hat M}_{\bfy}^\top {\bf\hat M}_{\bfy}(\bfL_{\bfy\bfy} + \lambda N{\bf I}_N)^{-1}\bfPsi_{\bfy}^\top\bfPsi_{\tilde\bfy} \\
    & = {\bf A}^\top {\bf\hat M}_{\bfy}^\top {\bf\hat M}_{\bfy} {\bf A}
\end{align}
where with similar notation abuse
\begin{equation}
    "{\bf\hat M}_{\bfy}^\top {\bf\hat M}_{\bfy}" = \left[\langle \hat\mu_{X|Y=y_i}, \hat\mu_{X|Y=y_j}\rangle_k\right]_{1\leq i, j\leq N} = \left[\frac{1}{n_i n_j}\sum_{l=1}^{n_i}\sum_{r=1}^{n_j}k(x_i^{(l)}, x_j^{(r)})\right]_{1\leq i, j\leq N}
\end{equation}

Substituting the latters into (\ref{eq:appendix-deconditional-mean}) and (\ref{eq:appendix-deconditional-covar}), we obtain empirical estimates of the deconditional posterior with shrinkage CMO estimator defined as
\begin{align}
     \,\!^S\!\hat m_\rmd(x) := m(x) + k_x^\top {\bf\hat M}_{\bfy}{\bf A}({\bf A}^\top {\bf\hat M}_{\bfy}^\top {\bf\hat M}_{\bfy} {\bf A} + \sigma^2{\bf I}_M)^{-1}(\tilde\bfz - \hat\mu(\tilde\bfy)), \\
     \,\!^S\!\hat k_\rmd(x, x') := k(x, x') - k_x^\top {\bf\hat M}_{\bfy}{\bf A}({\bf A}^\top {\bf\hat M}_{\bfy}^\top {\bf\hat M}_{\bfy} {\bf A} + \sigma^2{\bf I}_M)^{-1}{\bf A}^\top {\bf\hat M}_{\bfy}^\top k_{x'}
\end{align}
for any $x, x'\in\cX$.

Note that as the number of bags increases, it is possible to derive a variational formulation similar to the one proposed in Section~\ref{appendix:variational} that leverages the shrinkage estimator to further speed up the overall computation.

\newpage
\subsection{Ablation Study}
In this section we will present an ablation study on the shrinkage CMO estimator. The key is to illustrate that the Shrinkage CMO performs on par with the standard CMO estimator but is much faster to compute. 

In the following, we will sample bag data of the form $^b\!\cD = \{^b\boldx_j, y_j\}_{j=1}^N$ and $^b\boldx_j = \{x_{j}^{(i)}\}_{i=1}^n$, i.e there are $N$ bags with $n$ elements inside each. We first sample $N$ bag labels $y_j \sim \cN(0, 2)$ and for each bag $y_j$, we sample $n$ observations $x_{j}^{(i)}|y_j \sim \cN(y_j \sin(y_j), 0.5^2)$. 

Recall in standard CME one would need to repeat the number of bag labels to match the cardinality of $x_{j}^{(i)}$, i.e estimating CME using data $\{x_{j}^{(i)}, y_j\}_{j=1, i=1}^{N, n}$.

Denote $\hat{C}_{X|Y}$ as the standard CMO estimator and $^S\hat{C}_{X|Y}$ as the shrinkage CMO estimator. We will compare the RMSE between the two estimator when tested on a grid of test points $\{x^*_i, y^*_i\}_{i=1}^{N^*}$, i.e comparing the RMSE of the values between $\hat{\mu}_{X|Y=y^*_i}(x^*_i):=\langle \hat{C}_{X|Y}\ell_{y^*_i}, k_{x^*_i}\rangle_{k} $ and $^S\hat{\mu}_{X|Y=y^*_i}(x^*_i) := \langle ^S\hat{C}_{X|Y}\ell_{y^*_i}, k_{x^*_i}\rangle_{k}$ for each $i$. We also report the time in seconds needed to compute the estimator. The following results are ran on a CPU. Kernel hyperparameters are chosen using the median heuristic. The regularisation for both estimator is set to $0.1$.

\begin{figure}[!htp]
    \centering
    \includegraphics[width=\textwidth]{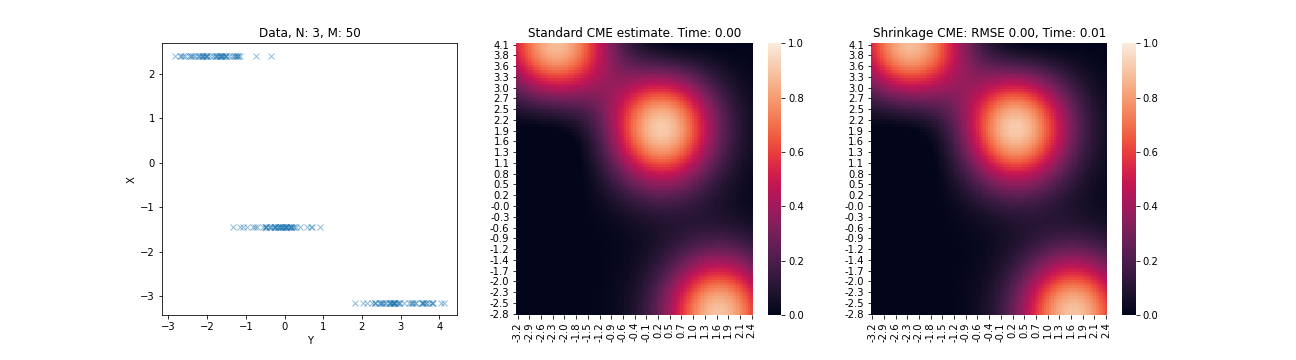}
    \caption{$3$ bags with $50$ samples each. (left) Data, (middle) $\hat{\mu}_{X|Y=y^*_i}(x^*_i)$ Standard CME. (right) $^S\hat{\mu}_{X|Y=y^*_i}(x^*_i)$ Shrinkage CME. We see both algorithms require very little time to train, ($\sim 0.01 $second) with a negligible difference in values as shown by the RMSE.}
    \label{fig:n3m50}
\end{figure}

\begin{figure}[!htp]
    \centering
    \includegraphics[width=\textwidth]{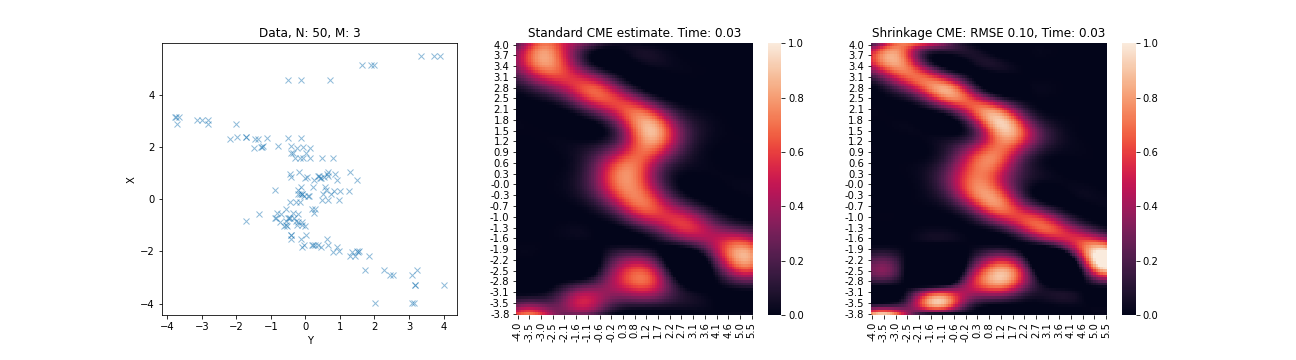}
    \caption{$50$ bags with $3$ samples each. (left) Data, (middle) $\hat{\mu}_{X|Y=y^*_i}(x^*_i)$ Standard CME. (right) $^S\hat{\mu}_{X|Y=y^*_i}(x^*_i)$ Shrinkage CME. Again, we see both algorithms require very little time to train, ($\sim 0.03 $ second). However, there is an increase in RMSE for the shrinkage estimator because there are much less samples for each bag, thus the empirical CME estimate $\hat{\mu}_{X|Y=y_j}$ might not be accurate. Nonetheless, it is still a small difference.}
    \label{fig:n50m3}
\end{figure}

Figures \ref{fig:n3m50} and \ref{fig:n50m3} show how shrinkage CMO performed compared to the standard CMO in a small data regime. Now when we increase the data size, we will start to see the major computational differences. (See Figures \ref{fig:n50m500} and \ref{fig:n500m50})
 
\begin{figure}[!htp]
    \centering
    \includegraphics[width=\textwidth]{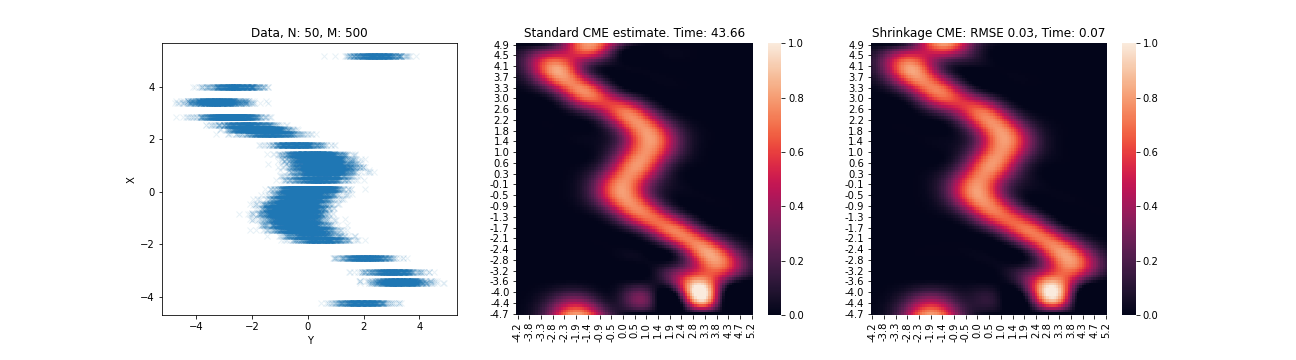}
    \caption{$50$ bags with $500$ samples each. (left) Data, (middle) $\hat{\mu}_{X|Y=y^*_i}(x^*_i)$ Standard CME. (right) $^S\hat{\mu}_{X|Y=y^*_i}(x^*_i)$ Shrinkage CME. With a small RMSE of $0.03$, the Shrinkage CME is approximately 600 times quicker than the standard version. }
    \label{fig:n50m500}
\end{figure}

\begin{figure}[!htp]
    \centering
    \includegraphics[width=\textwidth]{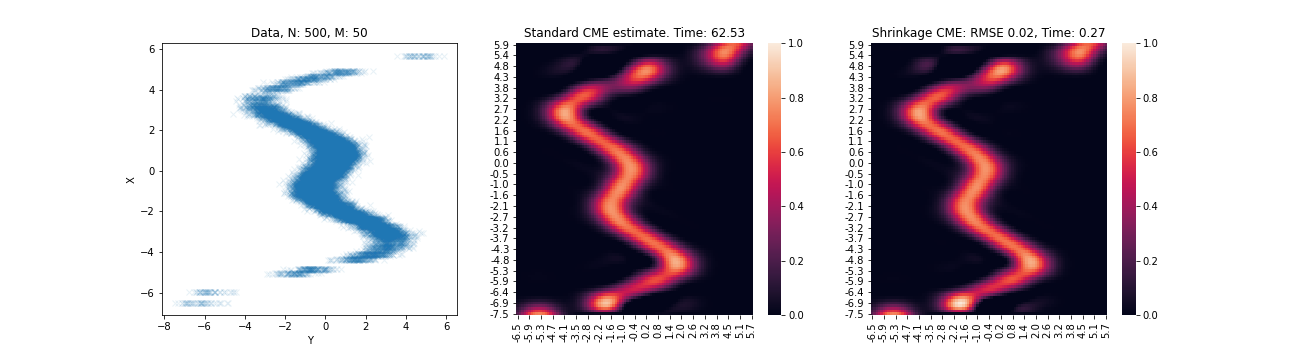}
    \caption{$500$ bags with $50$ samples each. (left) Data, (middle) $\hat{\mu}_{X|Y=y^*_i}(x^*_i)$ Standard CME. (right) $^S\hat{\mu}_{X|Y=y^*_i}(x^*_i)$ Shrinkage CME. Again, with a small RMSE of $0.02$, Shrinakge CME is  approximately 200 times quicker than the standard CME.}
    \label{fig:n500m50}
\end{figure}
\newpage

\strut
\newpage
\section{Details on Convergence Result}
\label{appendix: details on convergence result}

In this section, we provide insights about the convergence results stated in Section~\ref{section:deconditioning-as-regression}. These results are largely based on the impactful work of \citet{caponnetto2007optimal}, \citet{Szabo2016} and \citet{singh2019kernel} which we modify to fit our problem setup. Each assumption that we make is adapted from a similar assumption made in those works, for which we provide intuition and a detailed justification. We start by redefining the mathematical tools introduced in these works that are necessary to state our result.

\subsection{Definitions and $\cP_K(b, c)$ spaces}
We start by providing a general definition of covariance operators over vector-valued RKHS, which will allow us to specify a class of probability distributions for our convergence result.
\begin{definition}[Covariance operator]\label{definition:T}
    Let $\cW$ a Polish space endowed with measure $\rho$, $\cG$ a real separable Hilbert space and $K:\cW^2\rightarrow\cL(\cG)$ an operator-valued kernel spanning a $\cG$-valued RKHS $\cH_K$.
    
    The covariance operator of $K$ is defined as the positive trace class operator given by
    \begin{equation}
        T_K := \int_{\cZ} K_w K_w^* d\rho(w)\in\cL(\cH_K)
    \end{equation}
where $\cL(\cH_K)$ denotes the space of bounded linear operators over $\cH_k$.
\end{definition}

\begin{definition}[Power of self-adjoint Hilbert operator]
    Let $T$ a compact self-adjoint Hilbert space operator with spectral decomposition $T = \sum_{n=1}^\infty \lambda_n e_n\otimes e_n$ on $(e_n)_{n\in\NN}$ basis of $\operatorname{Ker}(T)^\perp$. The $r$\textsuperscript{th} power of $T$ is defined as $T^r =\sum_{n=1}^\infty \lambda_n^r e_n\otimes e_n$.
\end{definition}

Using the covariance operator, we now introduce  a general class of priors that does not assume parametric distributions, by adapting to our setup a definition originally introduced by \citet{caponnetto2007optimal}. This class captures the difficulty of a regression problem in terms of two simple parameters, $b$ and $c$~\cite{Szabo2016}.

\begin{definition}[$\cP_K(b, c)$ class]\label{definition:pbc-class}
    Let $\cE_\rho: \cG^\cZ\rightarrow[0, \infty[$ an expected risk function over $\rho$ and $E_\rho = \arg\min\,\cE_\rho$. Then given $b > 1$ and $c\in]1, 2]$, we say that $\rho$ is a $\cP_K(b, c)$ class probability measure w.r.t.\ $\cE_\rho$ if
    \begin{enumerate}
        \item \emph{Range assumption: } $\exists G\in\cH_K$ such that $E_\rho = T_K^{\frac{c-1}{2}}\circ G$ with $\|G\|_K^2\leq R$ for some $R\geq 0$
        \item \emph{Spectral assumption: } the eigenvalues $(\lambda_n)_{n\in\NN}$ of $T_K$ satisfy $\alpha\leq n^b\lambda_n\leq\beta\,,\forall n\in\NN$ for some $\beta\geq\alpha\geq 0$
    \end{enumerate}
\end{definition}

The range assumption controls the functional smoothness of $E_\rho$ as larger $c$ corresponds to increased smoothness. Specifically, elements of $\Range{T_K^{\frac{c - 1}{2}}}$ admit Fourier coefficients $(\gamma_n)_{n\in\NN}$ such that $\sum_{n=1}^\infty \gamma_n^2\lambda_n^{-(c+1)} < \infty$. In the limit $c\rightarrow 1$, we obtain $\Range{T_K^0} = \Range{\operatorname{Id}_{\cH_K}} = \cH_K$. Since ranked eigenvalues are positive and $\lambda_n\rightarrow 0$, greater power of the covariance operator $T_K$ give rise to faster decay of the Fourier coefficients and hence smoother operators.

The spectral assumptions can be read as a polynomial decay over the eigenvalues of $T_K$. Thus, larger $b$ leads to enhanced decay $\lambda_n = \Theta(n^{-b})$ and concretely in a smaller effective input dimension.

\subsection{Complete statement of the convergence result}

The following result corresponds to a detailed version of Theorem~\ref{theorem:convergence-rate} where all the assumptions are explicitly stated. As such, its proof also constitutes the proof for Theorem~\ref{theorem:convergence-rate}.
\begin{theorem}[Empirical DMO Convergence Rate]\label{theorem:detailed-convergence-result}
    Assume that
    \begin{enumerate}
        \item $\cX$ and $\cY$ are Polish spaces, i.e.\ separable and completely metrizable topoligical spaces
        \item $k$ and $\ell$ are continuous, bounded, their canonical feature maps $k_x$ and $\ell_y$ are measurable and $k$ is characteristic
        \item $\cH_\ell$ is finite dimensional
        \item $\arg\min\cE_\rmc \in\cH_\Gamma$ and $\arg\min\cE_\rmd\in\cH_\Gamma$
        \item The operator family $\{\Gamma_{\mu_{X|Y=y}}\}_{y\in\cY}$ is H\"{o}lder continuous with exponent $\iota \in]0, 1]$
        \item $\PP_{XY}$ is a $\cP_\Gamma(0, c')$ class probability measure w.r.t.\ $\cE_\rmc$ and $\PP_{Y}$ is a $\cP_\Gamma(b, c)$ class probability measure w.r.t.\ $\cE_\rmd$
        \item $\forall g\in\cH_\ell$, $\|g\|_\ell < \infty$ almost surely
    \end{enumerate}
    
    Let $D_{\PP_Y} = \underset{D\in\cH_\Gamma}{\arg\min}\,\cE_\rmd(D)$. Then, if we choose $\lambda = N^{-\frac{1}{c' + 1}}$ and $N = M^{\frac{a(c' + 1)}{\iota(c' - 1)}}$ where $a > 0$, we have
    \begin{itemize}
        \item If $a\leq\frac{b(c+1)}{bc+1}$, then $\cE_\rmd(\hat D_{X|Y}) - \cE_\rmd(D_{\PP_Y}) = \cO(M^\frac{-ac}{c+1})$  with $\epsilon=M^\frac{-a}{c+1}$
        \item If $a\geq\frac{b(c+1)}{bc+1}$, then $\cE_\rmd(\hat D_{X|Y}) - \cE_\rmd(D_{\PP_Y}) = \cO(M^\frac{-bc}{bc+1})$  with $\epsilon=M^\frac{-b}{bc+1}$
    \end{itemize}
\end{theorem}
\begin{proof}[Proof of Theorem~\ref{theorem:convergence-rate}]

The main objective here will be to rigorously verify that within our setup, the conditions in Theorem 4 from \cite{singh2019kernel} are met. We reformulate from our problem perspective each of the assumptions stated by \citet{singh2019kernel} and verify they are satisfied.

\paragraph*{Assumption 1} \emph{Assume observation model $\tilde Z = f(X) + \tilde\varepsilon$, with $\EE[\tilde\varepsilon|Y] = 0$ and suppose $\PP_{X|Y=y}$ is not constant in $y$.}

In this work, the observation model considered is $Z = \EE[f(X)|Y] + \varepsilon$ and the objective is to recover the underlying random variable $f(X)$ which noisy conditional expectation is observed. The latter presumes that we could bring $Z$ to $X$'s resolution. We can model it by introducing \enquote{pre-aggregation} observation model $\tilde Z = f(X) + \tilde\varepsilon$ such that $Z = \EE[\tilde Z|Y]$ and $\tilde\varepsilon$ is a noise term at individual level satisfying $\EE[\tilde\varepsilon|Y] = 0$.

\paragraph*{Assumption 2} \emph{$\cX$ and $\cY$ are Polish spaces}.

We also make this assumption.

\paragraph*{Assumption 3} \emph{$k$ and $\ell$ are continuous and bounded, their canonical feature maps are measurable and $k$ is characteristic}.

We make the same assumptions. The separability of $\cX$ and $\cY$ along with continuity assumptions on kernels allow to propagate separability to their associated RKHS $\cH_k$ and $\cH_\ell$ and to the vector-valued RKHS $\cH_\Gamma$. Boundedness and continuity on kernels ensure the measurability of the CMO and hence that measures on $\cX$ and $
cY$ can be extended to $\cH_k$ and $\cH_\ell$. The assumption on $k$ being characteristic ensures that conditional mean embeddings $\mu_{X|Y=y}$ uniquely embed conditional distributions $\PP_{X|Y=y}$ and henceforth operators over $\cH_\ell$ are identified.

\paragraph*{Assumption 4}\emph{$\arg\min\,\cE_\rmc\in\cH_\Gamma$}.

This property stronger is than what the actual conditional mean operator needs to satisfy, but it is necessary to make sure the problem is well-defined. We also make this assumption.

\paragraph*{Assumption 5}\emph{$\PP_{XY}$ is a $\cP_\Gamma(0, c')$ class probability measure, with $c'\in]1, 2]$}

As explained by \citet{singh2019kernel}, this is further required to bound the approximation error which we also make. Through the definition of the $\cP_\Gamma(0, c')$ class, this hypothesis assumes the existence of a probability measure over $\cH_k$ we denote $\PP_{\cH_k}$. Since $\cH_k$ is Polish (proof below), the latter can be constructed as an extension of $\PP_X$ over the Borel $\sigma$-algebra associated to $\cH_k$~\cite[Lemma A.3.16]{steinwart2008svm}.

\paragraph*{Assumption 6}\emph{$\cH_k$ is a Polish space}

Since $k$ is continuous and $\cX$ is separable, $\cH_k$ is a separable Hilbert space which makes it Polish.

\paragraph*{Assumption 7}\emph{The $\{\Gamma_{\mu_{X|Y=y}}\}_{y\in\cY}$ operator family is}
\begin{itemize}
    \item Uniformly bounded in Hilbert-Schmidt norm, i.e.\ $\exists B > 0$ such that \ $\forall y\in\cY$, $\|\Gamma_{\mu_{X|Y=y}}\|^2_{\operatorname{HS}(\cH_\ell, \cH_\Gamma)} \leq B$
    \item H\"older continuous in operator norm, i.e.\ $\exists L > 0, \iota\in]0, 1]$ such that $\forall y, y'\in\cY$, $\|\Gamma_{\mu_{X|Y=y}} - \Gamma_{\mu_{X|Y=y'}}\|_{\cL(\cH_\ell, \cH_\Gamma)}\leq L\|\mu_{X|Y=y} - \mu_{X|Y=y'}\|^\iota_{k}$
\end{itemize}
where $\cL(\cH_\ell, \cH_\Gamma)$ denotes the space of bounded linear operator between $\cH_\ell$ and $\cH_\Gamma$.

Since we assume finite dimensionality of $\cH_\ell$, we make a stronger assumption than the boundedness in Hilbert-Schmidt norm which we obtain as
\begin{align}
    \|\Gamma_{\mu_{X|Y=y}}\|^2_{\operatorname{HS}(\cH_\ell, \cH_\Gamma)} & = \tr\left(\Gamma(\mu_{X|Y=y}, \mu_{X|Y=y})\right) \\ 
    & = \tr\left(\langle \mu_{X|Y=y}, \mu_{X|Y=y}\rangle_k\operatorname{Id}_{\cH_\ell}\right) \\
    & = q(y, y)\tr\left(\operatorname{Id}_{\cH_\ell}\right) < \infty.
\end{align}

H\"older continuity is a mild assumption commonly satisfied as stated in ~\cite{Szabo2016}.

\paragraph*{Assumption 8}\emph{$\arg\min\,\cE_\rmd\in\cH_\Gamma$ and $\cH_\ell$ is a space of bounded functions almost surely}

We assume that the true minimiser of $\cE_\rmd$ is in $\cH_\Gamma$ to have a well-defined problem. The second assumption here is expressed in terms of probability measure $\PP_{\cH_\ell}$ over $\cH_\ell$. We do also assume that there exists $B > 0$ such that $\forall g\in\cH_\ell$, $\|g\|_\ell < B\enspace \PP_{\cH_\ell}-$ almost surely.

\paragraph*{Assumption 9}\emph{$\PP_Y$ is a $\cP_\Gamma(b, c)$ class probability measure, with $b > 1$ and $c\in]1, 2]$}

This last hypothesis is not required per se to obtain a bound on the excess error of regularized estimate $\hat D_{X|Y}$. However, it allows to simplify the bounds and state them in terms of parameters $b$ and $c$ which characterize efficient input size and functional smoothness respectively.

Furthermore, a premise to this assumption is the existence of a probability measure over $\cH_\ell$ that we denote $\PP_{\cH_\ell}$. Since $\ell$ is continuous and $\cY$ separable, it makes $\cH_\ell$ a separable and thus Polish. We can then construct $\PP_{\cH_\ell}$ by extension of $\PP_Y$~\cite[Lemma A.3.16]{steinwart2008svm}
\end{proof}

This theorem underlines a trade-off between the computational and statistical efficiency w.r.t.\ the datasets cardinalities $N = |\cD_1|$ and $M = |\cD_2|$ and the problem difficulty $(b, c, c')$.

For $a \leq \frac{b(c+1)}{bc+1}$, smaller $a$ means less samples from $\cD_1$ at fixed $M$ and thus computational savings. But it also hampers convergence, resulting in reduced statistical efficiency. At $a=\frac{b(c+1)}{bc+1} < 2$, convergence rate is a minimax computational-statistical efficiency optimal, i.e.\ convergence rate is optimal with smallest possible $M$. We note that at this optimal, $N > M$ and hence we require less samples from $\cD_2$. $a\geq\frac{b(c+1)}{bc+1}$ does not improve the convergence rate but increases the size of $\cD_1$ and hence the computational cost it bears.

We also note that larger H\"{o}lder exponents $\iota$, which translates in smoother kernels, leads to reduced $N$. Similarly, since $c'\mapsto\frac{c'+1}{c-1}$ and $c\mapsto\frac{b(c+1)}{bc+1}$ are strictly decreasing functions over $]1, 2]$, stronger range assumptions regularity which means smoother operators reduces the number of sample needed from $\cD_1$ to achieve minimax optimality. Smoother problems do hence require fewer samples.

Larger spectral decay exponent $b$ translate here in requiring more samples to reach minimax optimality and undermines optimal convergence rate. Hence problems with smaller effective input dimension are harder to solve and require more samples and iterations.

\newpage
\section{Additional Experimental Results}\label{appendix:section:experiments}

\subsection{Swiss Roll Experiment}\label{appendix:subsection:swiss-roll}
\subsubsection{Statistical significance table}
\begin{table}[!h]
\centering
\caption{p-values from a two-tailed Wilcoxon signed-rank test between all pairs of methods for the test RMSE of the swiss-roll experiment with a direct and indirect matching setup. The null hypothesis is that scores samples come from the same distribution. We only present the lower triangular matrix of the table for clarity of reading.}
\begin{tabular}{lrccccccc}
\toprule
   Matching &       &  \textsc{cmp} &    \textsc{bagg-gp} &    \textsc{varcmp} &     \textsc{vbagg} &     \textsc{gpr} &      \textsc{s-cmp}\\
\midrule
Direct  & \textsc{cmp} & - &  - &  - &  - &  - &  - \\
        & \textsc{bagg-gp} &  0.00006 & - &  - &  - &  - &  - \\
        & \textsc{varcmp} &  0.00008 &  0.00006 & - &  - &  - &  - \\
        & \textsc{vbagg} &  0.00006 &  0.00006 &  0.005723 & - &  - &  - \\
        & \textsc{gpr} &  0.00006 &  0.00006 &  0.00006 &  0.00006 & - &  - \\
        & \textsc{s-cmp} &  0.00006 &  0.00006 &  0.000477 &  0.014269 &  0.00006 & - \\ \midrule
Indirect & \textsc{cmp} & - &  - &  - &  - &  - &  -\\
         & \textsc{bagg-gp} &  0.011129 & - &  - &  - &  - & -\\
         & \textsc{varcmp} &  0.001944 &  0.015240 & - &  - &  - & -\\
         & \textsc{vbagg} &  0.000089 &  0.047858 &  0.000089 & - &  - & -\\
         & \textsc{gpr} &  0.025094 &  0.047858 &  0.047858 &  0.851925 & - & -\\
         & \textsc{s-cmp} &  0.000089 &  0.002821 &  0.000089 &  0.000140 &  0.052222 &  -\\
\bottomrule
\end{tabular}
\label{table:wilcoxon-swiss-roll}
\end{table}

\subsubsection{Compute and Resources Specifications}
Computations for all experiments were carried out on an internal cluster. We used a single GeForce GTX 1080 Ti GPU to speed up computations and conduct each experiment with multiple initialisation seeds. We underline however that the experiment does not require GPU acceleration and can be performed on CPU in a timely manner.

\newpage
\subsection{CMP with high-resolution noise observation model}

\subsubsection{Deconditional posterior with high-resolution noise}

Beyond observation noise on the aggregate observations $\tilde{\bfz}$ as introduced in Section~\ref{sec: deconditional posterior}, it is natural to also consider observing noise at the high-resolution level, i.e.\ noises placed on $f$ level directly in addition to the one $g$ at aggregate level. Let $\xi\sim\cG\cP(0, \delta)$ the zero-mean Gaussian process with covariance function
\begin{equation}
    \deffunction{\delta}{\cX\times\cX}{\RR}{(x, x')}{\begin{cases}1\text{ if } x = x' \\ 0\text{ else}\end{cases}}.
\end{equation}

By incorporating this gaussian noise process in the integrand, we can replace the definition of the CMP by
\begin{equation}
    g(y) = \int_{\cX}\left(f(x) + \varsigma\xi(x)\right) \d\PP_{X|Y=y}\,,\enspace\forall y\in\cY,
\end{equation}
where $\varsigma > 0$ is the high-resolution noise standard deviation parameter. Essentially, this amounts to consider a contaminated covariance for the HR observation process. This covariance is defined as
\begin{equation}
    \deffunction{k^\varsigma}{\cX\times\cX}{\RR}{(x, x')}{k(x, x') + \varsigma^2\delta(x, x')}.
\end{equation}
Provided the same regularity assumptions as in Proposition~\ref{proposition:cmp-characterization}, the covariance of the CMP becomes $q(y, y') = \EE[k^\varsigma(X, X')|Y=y, Y'=y']$ --- the mean and cross-covariance terms are not affected. Similarly be written in terms of conditional mean embeddings, but using as an integrand for the CMEs the canonical feature maps induced by $k^\varsigma$, i.e.\ $\mu_{X|Y=y}^\varsigma := \EE[k^\varsigma(X, \cdot)|Y=y]$ for any $y\in\cY$. Critically, this is reflected in the expression of the empirical CMP covariance which writes
\begin{equation}
    \hat q(y, y') = \ell(y, \bfy)(\bfL_{\bfy\bfy} + N\lambda{\bf I}_N)^{-1}(\bfK_{\bfx\bfx} + \varsigma^2{\bf I}_N)(\bfL_{\bfy\bfy} + N\lambda{\bf I}_N)^{-1}\ell(\bfy, y')
\end{equation}
thus, yielding matrix form
\begin{align}
    \hat\bfQ_{\tilde\bfy\tilde\bfy} & := \hat q(\tilde\bfy, \tilde\bfy) \\
    & = \bfL_{\tilde\bfy\bfy}(\bfL_{\bfy\bfy} + N\lambda{\bf I}_N)^{-1}(\bfK_{\bfx\bfx} + \varsigma^2{\bf I}_N)(\bfL_{\bfy\bfy} + N\lambda{\bf I}_N)^{-1}\bfL_{\bfy\tilde\bfy} \\
    & = {\bf A}^\top(\bfK_{\bfx\bfx} + \varsigma^2{\bf I}_N){\bf A}.
\end{align}
which can readily be used in (\ref{eq:deconditional-posterior-mean}) and (\ref{eq:deconditional-posterior-covariance}) to compute the deconditional posterior.

This high-resolution noise term introduces an additional regularization to the model that helps preventing degeneracy of the deconditional posterior covariance. Indeed, we have
\begin{align}
    \hat k_\rmd(\bfx, \bfx) & = \bfK_{\bfx\bfx} - \bfK_{\bfx\bfx}{\bf A}(\hat\bfQ_{\tilde\bfy\tilde\bfy} + \sigma^2{\bf I}_M)^{-1}{\bf A}^\top\bfK_{\bfx\bfx} \\
    & = \bfK_{\bfx\bfx} - \bfK_{\bfx\bfx}{\bf A}({\bf A}^\top(\bfK_{\bfx\bfx} + \varsigma^2{\bf I}_N){\bf A} + \sigma^2{\bf I}_M)^{-1}{\bf A}^\top\bfK_{\bfx\bfx} \\
    & = \bfK_{\bfx\bfx} - \bfK_{\bfx\bfx}({\bf A}{\bf A}^\top(\bfK_{\bfx\bfx} + \varsigma^2{\bf I}_N) + \sigma^2{\bf I}_M)^{-1}({\bf A}{\bf A}^\top\bfK_{\bfx\bfx}).\label{eq:degenerates}
\end{align}
where on the last line we have used the Woodburry identity. We can see that when $\sigma = \varsigma = 0$, (\ref{eq:degenerates}) degenerates to $0$. The aggregate observation model noise $\sigma$ provides a first layer of regularization at low-resolution. The high-resolution noise $\varsigma$ supplements it, making for a more stable numerical compuation for the empirical covariance matrix.

\subsubsection{Variational deconditional posterior with high-resolution noise}

The high-resolution noise observation process can also be incorporated into the variational derivation to obtain a slightly different ELBO objective. We have 
\begin{align}
     p({\bf\tilde z| f}) & = \cN({\bf\tilde z}| \boldsymbol{\Upsilon}^\top{\bf K}_{\bfx\bfx}^{-1}{\bf f}\;,\; {\bfQ_{\tilde\bfy\tilde\bfy}} + \sigma^2 {\bf I}_M - \boldsymbol{\Upsilon}^\top{\bf K_{\bfx\bfx}}^{-1}\boldsymbol{\Upsilon}) \\
     & = \cN({\bf\tilde z}| {\bf A}{\bf f}\;,\; {\bf A}^\top(\bfK_{\bfx\bfx} + \varsigma^2{\bf I}_N){\bf A} + \sigma^2 {\bf I}_M - {\bf A}^\top\bfK_{\bfx\bfx}{\bf A}) \\
     &= \cN({\bf\tilde z}| {\bf A}{\bf f}\;,\; \varsigma^2{\bf A}^\top {\bf A} + \sigma^2 {\bf I}_M)
\end{align}

The expected loglikelihood with respect to the variational posterior hence writes
\begin{align}
    \EE_{q({\bf f})}[p({\bf\tilde z| f})] = & -\frac{M}{2}\log(2\pi) -\frac{1}{2}\log\det(\varsigma^2{\bf A}^\top {\bf A}
    + \sigma^2{\bf I}_M)\\
    & - \frac{1}{2}\EE_{q(\bf f)}\left[(\tilde\bfz - {\bf A}^\top{\bf f})^\top(\varsigma^2{\bf A}^\top {\bf A}+ \sigma^2{\bf I}_M)^{-1}(\tilde\bfz - {\bf A}^\top{\bf f})\right]
\end{align}

With a derivation similar to the one proposed in Appendix~\ref{appendix:variational}, the expected loglikelihood can be expressed in terms of the posterior variational parameters as
\begin{align}
    \EE_{q({\bf f})}[p({\bf\tilde z| f})] = & -\frac{M}{2}\log(2\pi) -\frac{1}{2}\log\det(\varsigma^2{\bf A}^\top {\bf A} + \sigma^2{\bf I}_M) \\
    & - \frac{1}{2}(\tilde\bfz - {\bf A}^\top{\bf\bar\bfeta})^\top(\varsigma^2{\bf A}^\top {\bf A}+ \sigma^2{\bf I}_M)^{-1}(\tilde\bfz - {\bf A}^\top{\bf\bar\bfeta}) \\
    & - \frac{1}{2}\tr\left((\varsigma^2{\bf A}^\top {\bf A}+ \sigma^2{\bf I}_M)^{-1}{\bf A}^\top {\bf\bar\bfSigma}{\bf A}\right)
\end{align}
In particular, the last term can be rearranged into $\tr\left({\bf\bar\bfSigma}^{\nicefrac{1}{2}}{\bf A}(\varsigma^2{\bf A}^\top {\bf A}+ \sigma^2{\bf I}_M)^{-1}{\bf A}^\top {\bf\bar\bfSigma}^{\nicefrac{1}{2}}\right)$ which can efficiently be computed as an inverse quadratic form~\cite{gardner2018gpytorch}.

\newpage
\subsection{Mediated downscaling of atmospheric temperature}

\subsubsection{Map visualization of atmospheric fields dataset}
\begin{figure}[h!]
    \centering
    \includegraphics[width=\linewidth]{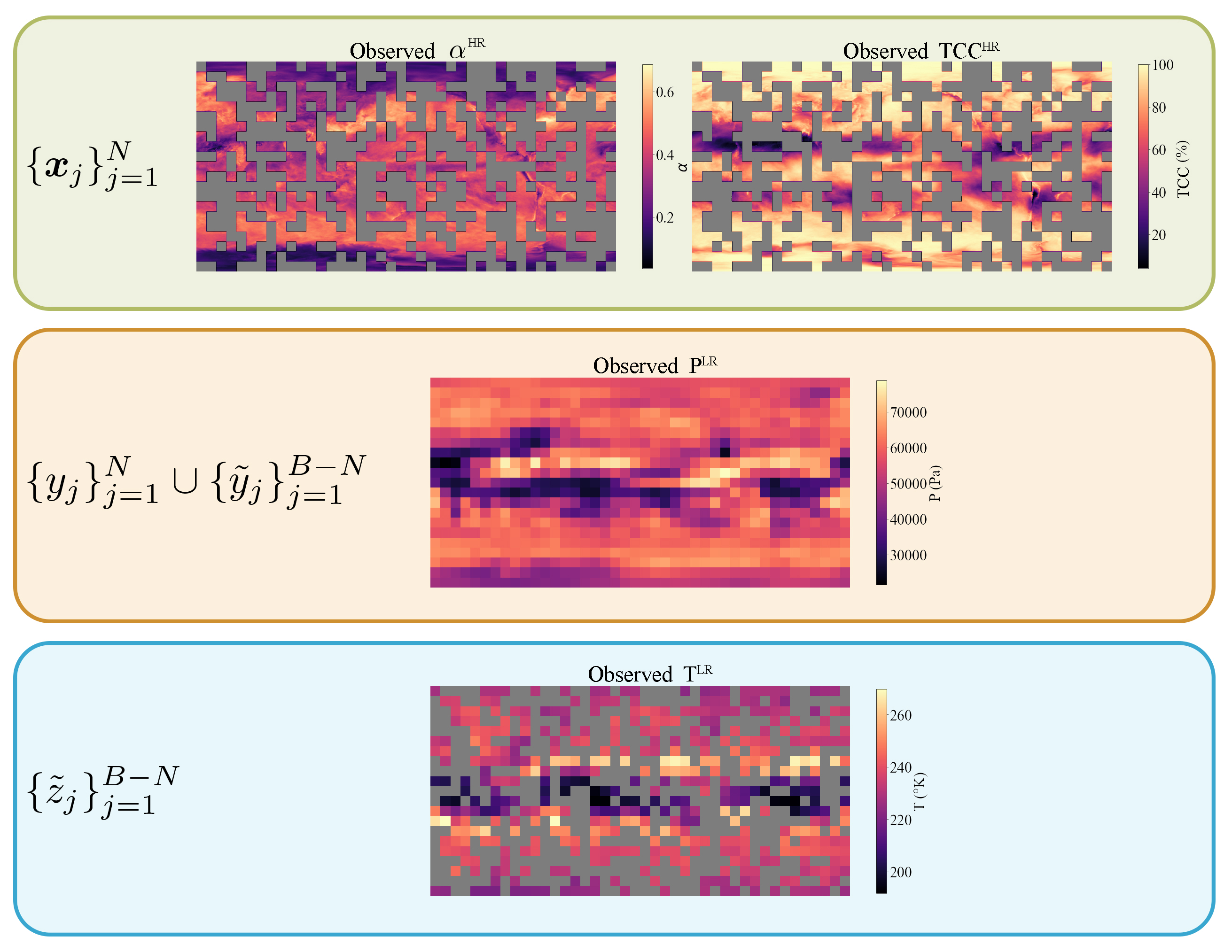}
    \caption{Map visualization of the dataset used in the mediated downscaling experiment (for one random seed); {\bf Top:} Bags of high-resolution albedo $\alpha$\textsuperscript{HR} and total cloud cover TCC\textsuperscript{HR} pixels which are observed in $\cD_1$ --- each \enquote{coarse pixel} delineates a bag of HR pixels; {\bf Middle:} Low-resolution pressure field P\textsuperscript{LR} which is observed everywhere and plays the role of mediating variable; {\bf Bottom:} Low-resolution temperature field T\textsuperscript{LR} pixels which are observed in $\cD_2$ and that we want to downscale; grey pixels are unobserved; the grey layer on HR covariates maps (top) is the exact complementary of the grey layer on the observed T\textsuperscript{LR} map (bottom).}
    \label{fig:downscaling-dataset}
\end{figure}

\newpage
\subsubsection{Downscaling prediction maps}
\begin{figure}[h!]
    \centering
    \includegraphics[width=\linewidth]{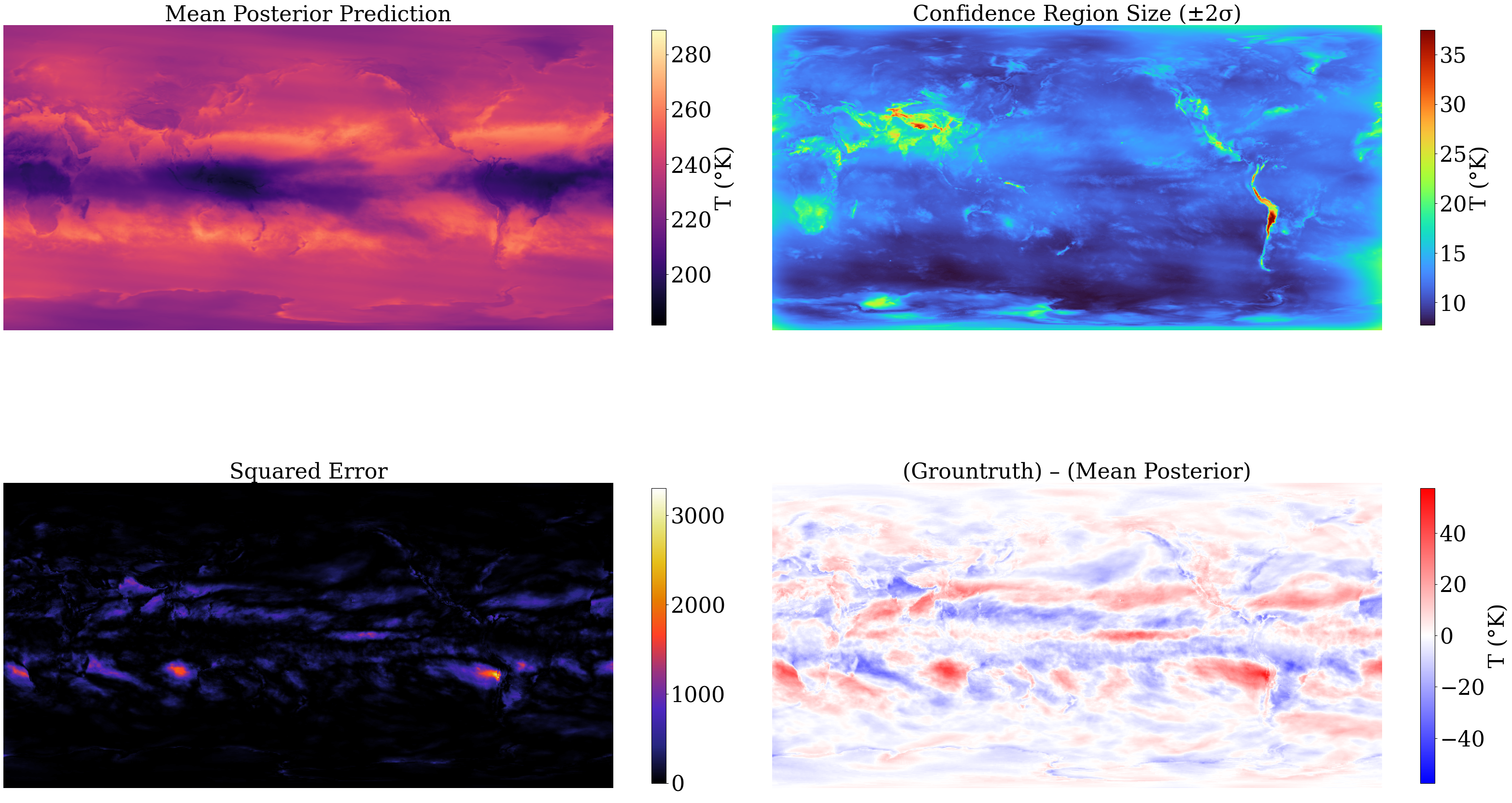}
    \caption{Predicted downscaled atmospheric temperature field with \textsc{vargpr}; {\bf Top-Left:} Posterior mean; {\bf Top-Right:} 95\% confidence region size, i.e.\ 2 standard deviation of the posterior; {\bf Bottom-Left:} Squared difference with unobserved groundtruth T\textsuperscript{HR}; {\bf Bottom-Right:} Difference between unobserved groundtruth T\textsuperscript{HR} and the posterior mean.}
    \label{fig:downscaling-prediction-gpr}
\end{figure}
\begin{figure}[h!]
    \centering
    \includegraphics[width=\linewidth]{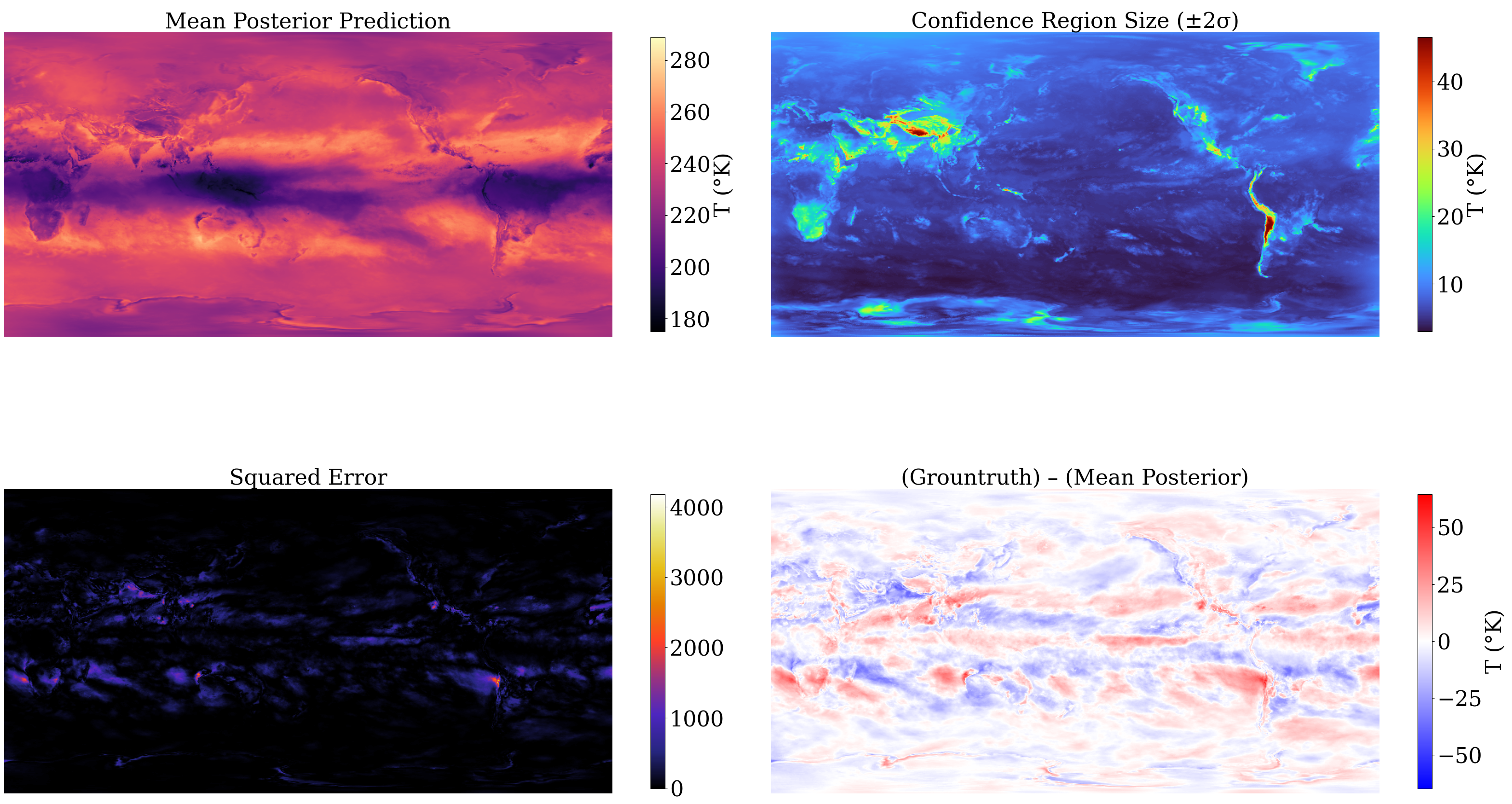}
    \caption{Predicted downscaled atmospheric temperature field with \textsc{vbagg}; {\bf Top-Left:} Posterior mean; {\bf Top-Right:} 95\% confidence region size, i.e.\ 2 standard deviation of the posterior; {\bf Bottom-Left:} Squared difference with unobserved groundtruth T\textsuperscript{HR}; {\bf Bottom-Right:} Difference between unobserved groundtruth T\textsuperscript{HR} and the posterior mean.}
    \label{fig:downscaling-prediction-vbagg}
\end{figure}
\begin{figure}[h!]
    \centering
    \includegraphics[width=\linewidth]{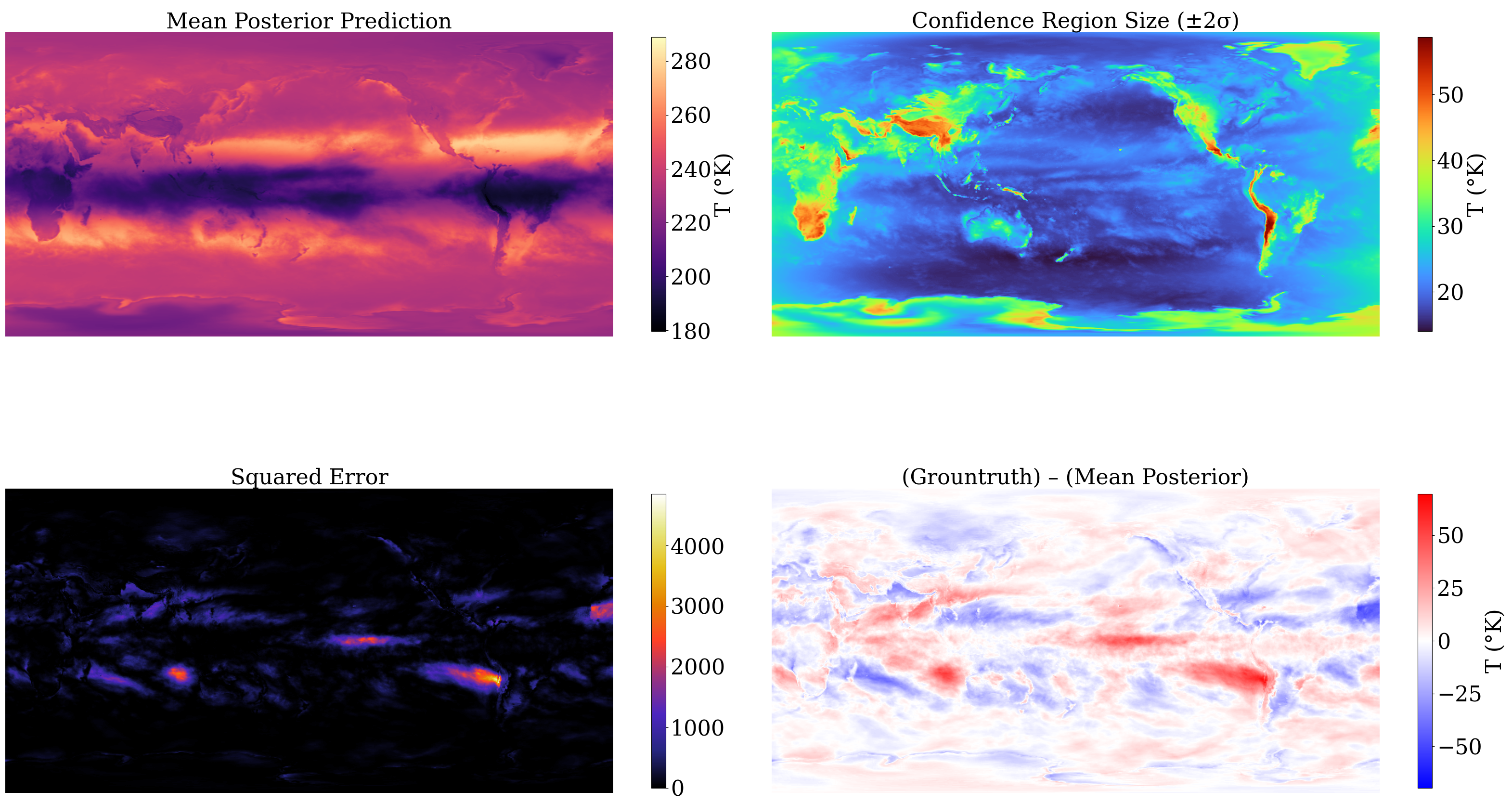}
    \caption{Predicted downscaled atmospheric temperature field with \textsc{varcmp}; {\bf Top-Left:} Posterior mean; {\bf Top-Right:} 95\% confidence region size, i.e.\ 2 standard deviation of the posterior; {\bf Bottom-Left:} Squared difference with unobserved groundtruth T\textsuperscript{HR}; {\bf Bottom-Right:} Difference between unobserved groundtruth T\textsuperscript{HR} and the posterior mean.}
    \label{fig:downscaling-prediction-varcmp}
\end{figure}

\newpage
\subsubsection{Statistical significance table}
\begin{table}[h!]
\centering
\caption{p-values from a two-tailed Wilcoxon signed-rank test between all pairs of methods for the evaluation scores on the mediated statistical downscaling experiment. The null hypothesis is that scores samples come from the same distribution. As before, we only present the lower-traingular table for clarity of reading.}
\begin{tabular}{lrccc}
\toprule
Metric  &       &    \textsc{varcmp} &     \textsc{vbagg} &     \textsc{vargpr} \\
\midrule
     & \textsc{varcmp} & - &  - &  - \\
RMSE & \textsc{vbagg} &  0.005062 & - &  - \\
     & \textsc{vargpr} &  0.006910 &  0.046853 & - \\\midrule
     & \textsc{varcmp} & - &  - &  - \\
MAE  & \textsc{vbagg} &  0.005062 & - &  - \\
     & \textsc{vargpr} &  0.059336 &  0.006910 & - \\\midrule
     & \textsc{varcmp} & - &  - &  - \\
CORR& \textsc{vbagg} &  0.005062 & - &  - \\
     & \textsc{vargpr} &  0.016605 &  0.028417 & - \\\midrule
     & \textsc{varcmp} & - &  - &  - \\
SSIM & \textsc{vbagg} &  0.005062 & - &  - \\
     & \textsc{vargpr} &  0.959354 &  0.005062 & - \\
\bottomrule
\end{tabular}

\label{table:wilcoxon-matched-swiss-roll-rmse}
\end{table}

\subsubsection{Compute and Resources Specifications}
Computations for all experiments were carried out on an internal cluster. We used a single GeForce GTX 1080 Ti GPU to speed up computations and conduct each experiment with multiple initialisation seeds.

\end{document}